\documentclass[10pt,journal,compsoc]{IEEEtran}
\usepackage{colortbl}
\usepackage{graphicx}
\usepackage{epstopdf}
\usepackage{amsmath,bm}
\usepackage{amsthm}
\usepackage{amssymb}
\usepackage{amsfonts}
\usepackage{algorithmic}
\usepackage{algorithm}
\usepackage{array}
\usepackage{multirow}
\usepackage{booktabs}
\usepackage{color}
\usepackage{float}
\usepackage{subfig}
\usepackage{makecell}
\usepackage{threeparttable} 
\usepackage{cases}
\usepackage{ragged2e}
\usepackage[colorlinks]{hyperref}
\usepackage{xr}
\usepackage{soul}
\usepackage[switch,pagewise]{lineno}
\soulregister\cite7
\soulregister\ref7 
\makeatletter

\newcommand*{\addFileDependency}[1]{
  \typeout{(#1)}
  \@addtofilelist{#1}
  \IfFileExists{#1}{}{\typeout{No file #1.}}
}

\makeatother

\newtheorem{theorem}{Theorem}
\newtheorem{proposition}[theorem]{Proposition}%
\newtheorem{remark}{Remark}%
\newtheorem{definition}{Definition}%

\begin{document}

\title{IG\textsuperscript{2}: Integrated Gradient on Iterative Gradient Path for Feature Attribution}
\author{Yue Zhuo, Zhiqiang Ge, \IEEEmembership{Senior Member,~IEEE}

  \IEEEcompsocitemizethanks{
    \IEEEcompsocthanksitem This paper has been accepted by IEEE Transactions on Pattern Analysis and Machine Intelligence.
    \IEEEcompsocthanksitem This work was supported in part by the National Natural Science Foundation of China (NSFC) (92167106, 62103362 and 61833014) and the Natural Science Foundation of Zhejiang Province (LR18F030001). \textit{(Corresponding author: Zhiqiang Ge)}
    \IEEEcompsocthanksitem Yue Zhuo is with the State Key Laboratory of Industrial Control Technology, College of Control Science and Engineering, Zhejiang University, Hangzhou, 310027, China. Zhiqiang Ge is with Peng Cheng Laboratory, Shenzhen, 518000, China. (E-mail: zhuoy1995@zju.edu.cn, zhiqiang.ge@hotmail.com)}}
    
\IEEEtitleabstractindextext{\begin{abstract}
    Feature attribution explains Artificial Intelligence (AI) at the instance level by providing importance scores of input features' contributions to model prediction. Integrated Gradients (IG) is a prominent path attribution method for deep neural networks, involving the integration of gradients along a path from the explained input (explicand) to a counterfactual instance (baseline). Current IG variants primarily focus on the gradient of explicand's output. However, our research indicates that the gradient of the counterfactual output significantly affects feature attribution as well. To achieve this, we propose \underline{I}terative \underline{G}radient path \underline{I}ntegrated \underline{G}radients (IG\textsuperscript{2}), considering both gradients. IG\textsuperscript{2} incorporates the counterfactual gradient iteratively into the integration path, generating a novel path (\emph{GradPath}) and a novel baseline (\emph{GradCF}). These two novel IG components effectively address the issues of attribution noise and arbitrary baseline choice in earlier IG methods. IG\textsuperscript{2}, as a path method, satisfies many desirable axioms, which are theoretically justified in the paper. Experimental results on XAI benchmark, ImageNet, MNIST, TREC questions answering, wafer-map failure patterns, and CelebA face attributes validate that IG\textsuperscript{2} delivers superior feature attributions compared to the state-of-the-art techniques. The code is released at: \url{https://github.com/JoeZhuo-ZY/IG2}.
  \end{abstract}
  }

\maketitle

\IEEEraisesectionheading{\section{Introduction}\label{intro}}

\par \IEEEPARstart{A}{I} models are becoming increasingly prevalent in critical fields, such as industrial control and biomedical analysis. Consequently, the need for research into their explainability (XAI) has become urgent. This is essential to keep humans in the loop and help people understand, explain, and control the models~\cite{surveyXAI,9241434,9335497}. Given an input instance (e.g., an image), feature attribution for deep neural networks quantifies the contributions of individual features, such as pixels, to the model output. These results can support the users in reasoning which input elements drive model predictions.

\par The gradient is a basic form of feature attribution that analogizes the model's coefficients for a deep network. Early local gradient methods such as Vanilla Gradient~\cite{saliencymaps1}, Grad-CAM~\cite{Gradcam}, and Guided Backpropagation~\cite{Guided_BP} suffer from the gradient saturation, a problem that the gradients in the input neighborhood are misleading~\cite{DeepLIFT,highcurvate}. Recently, for solving this problem, Integrated Gradients (IG)~\cite{IG} was proposed as a path method that integrates gradient (of model output) along a path between the explained instance (i.e., explicand) and baseline. 


\par IG methods introduce the concept of counterfactual explanation, which contrastively explain the models by answering: ``Which features cause the model output prediction A (of explicand) rather than counterfactual prediction B (of baseline)\footnotemark ?'' From the perspectives of philosophy and psychology, the counterfactuals align with human cognition to explain unexpected events~\cite{CFE,CF_GDPR}, and have been implied in many attribution methods, such as SHapley Additive exPlanations~\cite{SHAP}, DeepLIFT~\cite{DeepLIFT}, SCOUT~\cite{Wang_2020_CVPR}, and sub-region interpretation~\cite{chen2024Less}.
\footnotetext{Note that the terms \emph{counterfactuals}~\cite{CF_GDPR}, \emph{contrastive facts}~\cite{DBLP:journals/corr/abs-1802-07623}, and \emph{counter (class) example}~\cite{10.1145/3563039} are used interchangeably in prior explanation works. In the context of Shapley value~\cite{SHAP} and path methods~\cite{IG}, they are commonly referred to as \emph{background} and \emph{baseline}.}

\par The focus of this study is on IG, a well-known path (attribution) method for deep neural networks. Path methods, rooted in Aumann-Shapley game theory, adhere to many describable axioms~\cite{sundararajan2020many,IG}. IG's attribution performance depends on two essential hyperparameters: path and baseline. The conventional approach, Vanilla IG, typically employs a zero baseline (e.g., all-black image) along with a straight-line path. However, this choice is arbitrary and agnostic to the model and explicand, leading to several shortcomings. For instance, the straight-line path can introduce noise into the attribution due to the saturation effect~\cite{Saturation} and the use of a black baseline will result in incomplete attributions~\cite{EG,baselines}.

\begin{figure*}[t]
  \centering
  \includegraphics[width=0.80\textwidth]{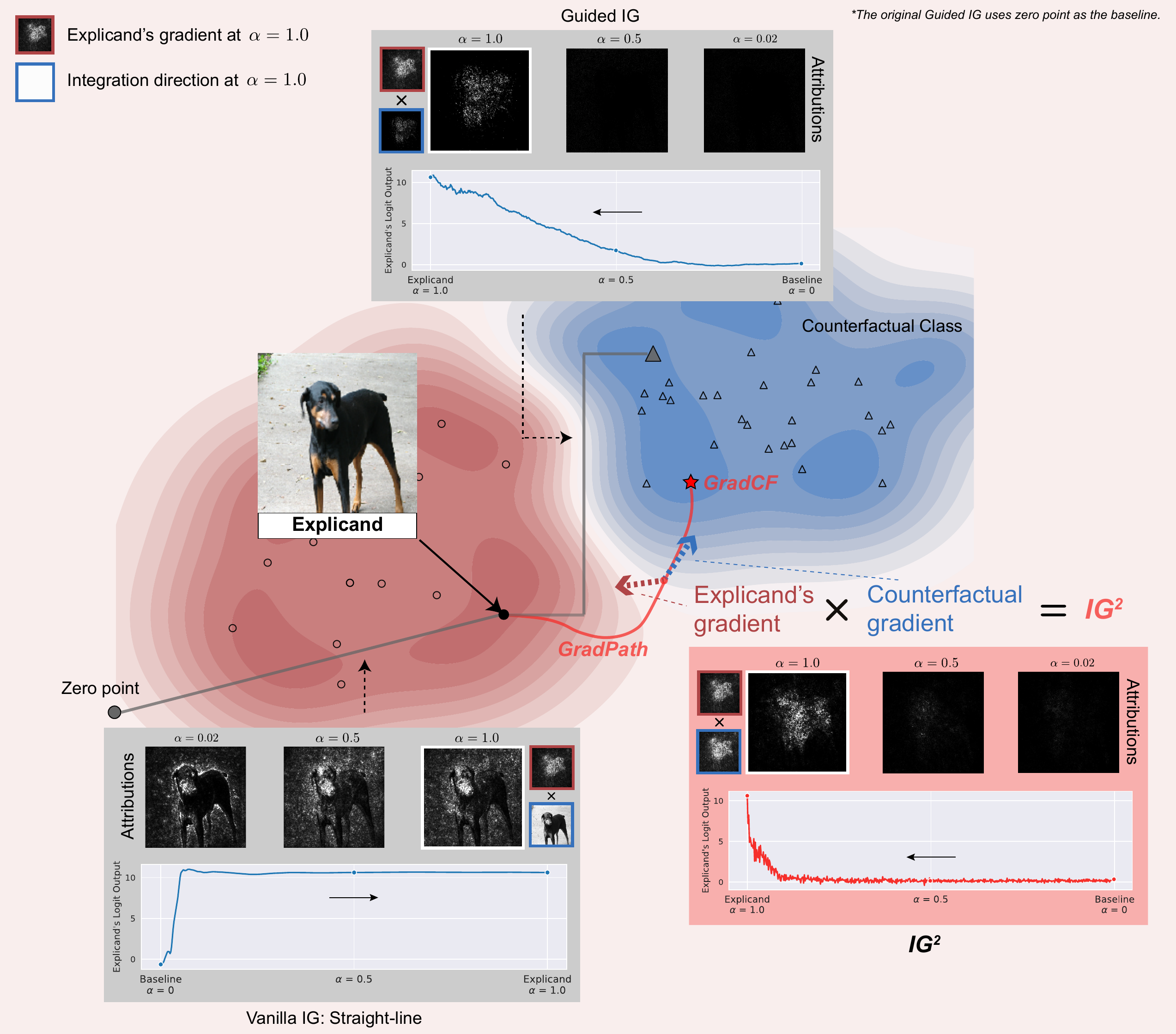}
  \caption{\textbf{Illustration of IG\textsuperscript{2} with GradPath and GradCF}, compared with Vanilla IG (with a zero baseline) and Guided IG (counterfactual baseline). The explicand is a sample [Doberman] from ImageNet~\cite{ILSVRC15} classified by Inception-v3~\cite{Inceptionv3}. The graphs of logit predictions for [Doberman] w.r.t. $\alpha$ value are plotted. The attributions snapshots of each method are shown at $\alpha$ values of 0.02, 0.5 and 1.0. At $\alpha=1.0$, the attributions are decomposed into two multipliers: explicand's gradient (small images in red boxes) and integration path direction (in blue boxes). Integration path of IG\textsuperscript{2} optimally aligns the explicand's gradient with the counterfactual gradient. This alignment results in feature attributions that are less noisy and more complete (on the body of Doberman). Comprehensive attribution results on ImageNet can be found in Section \ref{sec_image_exp}.}
  \label{Paths_comparison}
\end{figure*}

\par Recently, with different paths and baselines, many variant path methods have been proposed for improving attributions. For integration paths, Guided IG~\cite{guided} adaptively chooses the path by selecting features with the smallest partial derivatives; Blur IG~\cite{9157016} integrated the gradients on the gradually blurred image path. For baselines, Expected IG~\cite{EG} sampled the baselines from the data distribution; Sturmfels et al.~\cite{baselines} discussed different baselines' impacts on the path methods.
\par  Notably, despite these advancements, the majority of existing baselines and paths are considered model-agnostic and explicand-agnostic. We believe the excellent baseline and path should contain the information of both explicand and model, which motivates the idea of IG\textsuperscript{2}. Table \ref{Summary} summarizes the existing IG-based methods from paths and baselines, and Section \ref{relat_pathmethod} contrasts them with our proposal in detail.


\par \textbf{IG\textsuperscript{2}:} Fig. \ref{Paths_comparison} depicts IG\textsuperscript{2}. Starting from the explicand, IG\textsuperscript{2} iteratively searches the instances in the descent direction of the counterfactual gradient, minimizing the representation distance between the explicand and counterfactual. The set of all samples searched at each step is denoted as GradPath, and its endpoint is the baseline GradCF, a counterfactual (CF) example. As the name suggests, IG\textsuperscript{2} multiplies two gradients during the integration: one of explicand's prediction and another of the counterfactual class, the latter of which is implied in the GradPath. IG\textsuperscript{2} provides superior attributions over existing techniques, due to the distinctions on two essentials of path methods: the integration path and baseline.

\par GradPath in IG\textsuperscript{2} effectively mitigates saturation effects~\cite{guided,Saturation} (in Definition \ref{def_sat}). This is achieved by its alignment with the counterfactual gradients, leading to a rapid decrease in the model's prediction of the explicand. Fig. \ref{Paths_comparison} show the merit of GradPath by the attributions at $\alpha=1$, where three path share an identical explicand's gradient. The GradPath, by directing itself towards the salient features that distinguish the explicand and counterfactual, which ``filters'' the noise in explicand's gradient by path integration. In contrast, the straight-line in Vanilla IG is on the dissimilar direction to the explicand's gradient, causing the noise in image background to accumulate along the integration path, leading to saturation effects; Guided IG shares a similar idea with IG\textsuperscript{2} but its path direction is constrained, which results in less complete attributions than IG\textsuperscript{2}.

\par GradCF is a novel baseline proposed in IG\textsuperscript{2}, and its advantages are illustrated in Fig. \ref{baseline_comparison}, contrasting the explicand with different baselines. GradCF can significantly highlight the critical features. For [Doberman] example, GradCF accurately highlights the dog's body with the counterfactual gradients. On the contrary, using a black baseline disregards the black pixels on dog and adversely attributes to the white background instead. While Expected IG employs counterfactual data as the baseline to address this issue, but this naive contrast in the input space remains irrelevant to the explicand, resulting in noises~\cite{baselines}. For other two datasets, GradCF also precisely highlights the key features, distinguishing between digital 5 and digital 6, as well as the defective wafer map in contrast to normal ones.

\begin{figure}[t]
  \setlength{\belowcaptionskip}{-0.4cm}
  \centering
  \includegraphics[width=0.42\textwidth]{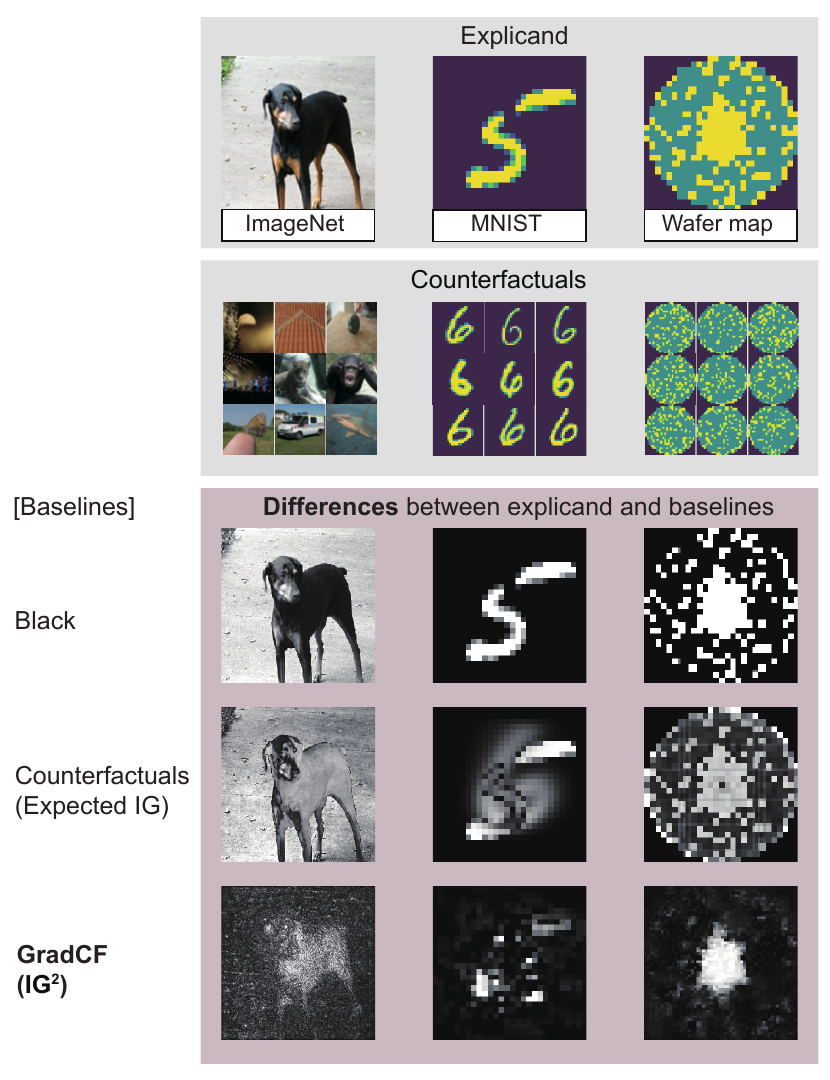}
  \caption{\textbf{Differences between explicand and three baselines}: the black image, counterfactuals and GradCF. Three samples from ImageNet, MNIST~\cite{6932449} and wafer map~\cite{6932449} datasets are plotted. For the non-black baselines of counter class, the ImageNet sample is contrasted with randomly sampled images, MNIST's digital 5 is contrasted with digitals 6, and the central failure wafer map is contrasted with normal wafer maps. The critical features in explicands are accurately highlighted by contrasting with GradCF. Detailed discussions of GradCF on MNIST examples are in Section \ref{interp_gradcf}.  }
  \label{baseline_comparison}
\end{figure}



\par In this study, our primary contribution is the introduction of IG\textsuperscript{2}, a novel path attribution method. IG\textsuperscript{2} comprises both a novel baseline, GradCF, and a novel integration path, GradPath. To the best of our knowledge, it is the first time that counterfactual gradients are integrated into the path attribution methods. Through extensive experiments on datasets from diverse domains, both qualitative and quantitative results consistently demonstrate IG\textsuperscript{2}'s superiority over existing state-of-the-art attribution methods. Furthermore, we substantiate the individual effectiveness of GradCF and GradPath through an ablation study.

\par The remainder of the paper is organized as: Section \ref{pre_path} introduces the preliminary about path methods; Section \ref{method} demonstrates our proposal in detail; Section \ref{Interpreting} gives a deep insight into IG\textsuperscript{2} and theoretically justify its axioms; Section \ref{relat} contrasts our method with related works across different fields; Section \ref{exper} presents experimental results to verify the attribution performances of IG\textsuperscript{2}; Section \ref{imple_detail} presents the IG\textsuperscript{2} implementation details; Section \ref{conclusion} concludes the paper.

\renewcommand{\arraystretch}{1.3}
\setlength\tabcolsep{3pt}
\begin{table}[]
  \centering
  \caption{Summary of existing path methods from the aspects of path and baseline}
  \label{Summary}
  \begin{threeparttable}
    \begin{tabular}{@{}llclcc@{}}
      \toprule
      Methods                               & Path                                                                 & M-s\tnote{$\dagger$} & Baseline           & E-s\tnote{$\S$} & M-s\tnote{$\dagger$} \\ \midrule

      \multirow{3}{*}{\cite{baselines}*}    & \multirow{3}{*}{straight line}                                       & \multirow{4}{*}{}    & maximal distance   & \checkmark      &                      \\
                                            &                                                                      &                      & noised data        & \checkmark      &                      \\
                                            &                                                                      &                      & uniform noise      &                 &                      \\

      Vanilla IG~\cite{IG}                  & straight line                                                        &                      & zero vector        &                 &                      \\
      Expected IG~\cite{EG}                 & straight line                                                        &                      & train data         &                 &                      \\
      XRAI~\cite{XRAI}                      & straight line                                                        &                      & black+white images &                 &                      \\
      Blur IG~\cite{9157016}                & blur path                                                            &                      & blurred image      & \checkmark      &                      \\
      Guided IG~\cite{guided}               & \begin{tabular}[c]{@{}l@{}}straight line's\\ projection\end{tabular} & \checkmark           & zero vector        &                 &                      \\
      \textbf{IG\textsuperscript{2} (ours)} & GradPath                                                             & \checkmark           & GradCF             & \checkmark      & \checkmark           \\ \bottomrule
    \end{tabular}

    \begin{tablenotes}
      \footnotesize
      \item[*] Work~\cite{baselines} discussed all the other baselines in Vanilla IG, Expected IG, XRAI, Blur IG, and Guided IG (except ours), which is not listed for clarity.
      \item[$\dagger$] M-s (Model-specific): The path (or baseline) is specifically designed for explained models.
      \item[$\S$] E-s (Explicand-specific): The baseline is specifically designed for explained sample.
    \end{tablenotes}
  \end{threeparttable}

\end{table}

\section{Preliminary of path methods}
\label{pre_path}
\par Path methods are based on the Aumann-Shapley theory from cost-sharing with many desirable properties. Given input instances of $n$ dimension $x \in \mathbb{R}^n$, the path of gradient integral is formally defined as $\gamma(\alpha)$ for $\alpha\in[0,1]$. Path $\gamma(\alpha):[0,1] \rightarrow \mathbb{R}^n$ consists of a set of points in $\mathbb{R}^n$, from the baseline $x'$ to the explicand $x$ (i.e., $\gamma(0)=x'$ and $\gamma(1)=x$).
\par Given a path $\gamma$ and model $f:\mathbb{R}^n \rightarrow \mathbb{R}$ (in classification models, only considering the output of explicand's class label), path integrated gradient attributes the $i^{th}$ feature by integrating the gradients of the model output w.r.t the $i^{th}$ feature value along the path $\gamma(\alpha)_i$, which is defined as~\cite{IG}:
\begin{align}
  \phi^{Path}_i = \int_{0}^{1} \frac{\partial f(\gamma(\alpha))}{\partial \gamma(\alpha)_i}  \frac{\partial \gamma(\alpha)_i}{\partial \alpha} d\alpha ,
  \label{PathIG}
\end{align}
where the first multiplier is the explicand's gradient of model prediction and the second one is the path direction.
\par \textbf{Baselines:} The choices of baseline $x'$ are various, and there is currently no consensus on which baseline is the best. Work~\cite{baselines} carefully researched the mainstream baselines, and we summarized them in Table \ref{Summary}. The zero (black) instance $x'=\vec{0}$, one (white) instance $x'=\vec{1}$ and train data $x' \sim D_{train}$ are three commonly used baselines.

\par \textbf{Straight-line path:} The commonly used path is the straight line from the explicand $x$ to baseline $x'$, which is specified $\gamma(\alpha)=x'+\alpha\times(x-x')$ for $\alpha\in[0,1]$.
\par In practice, it is intractable to directly compute the integral in Eq. \ref{PathIG}. Instead, a Riemann sum is used for a discrete approximate with $k$ points in the integral interval. On a straight-line path, the IG~\cite{IG} is computed by:
\begin{align}
  \phi^{IG}_i = (x_i-x'_i)\times \sum_{i=1}^k \frac{\partial f(x'+\frac{k}{m}\times (x-x'))}{\partial x_i} \times \frac{1}{m}.
  \label{Riemann}
\end{align}


\begin{figure}[!t]
  \centering
  \includegraphics[width=0.45\textwidth]{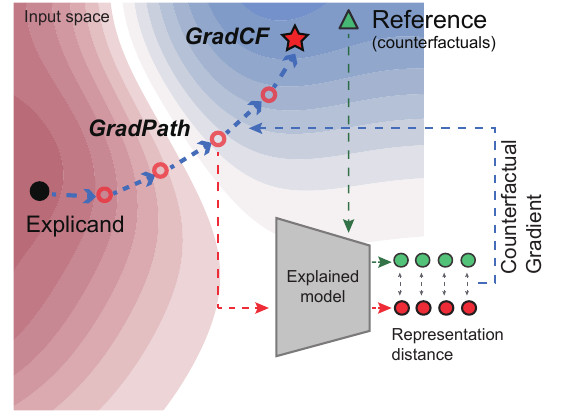}
  \caption{\textbf{Illustration for building GradPath} at each step. From the explicand, the direction of GradPath is iteratively built on the gradient direction for minimizing the model representation distance to the reference. }
  \label{framework}
\end{figure}

\section{Methodologies of IG\textsuperscript{2}}
\label{method}
\par IG\textsuperscript{2} is a path method that extends IG~\cite{IG} by introducing a novel baseline (GradCF) and a novel integration path (GradPath). As the name suggests, IG\textsuperscript{2} accumulates not only the gradient of the explicand's prediction but also the counterfactual gradient, the latter of which is contained in the direction of GradPath. 
\par Given an explicand $x$, IG\textsuperscript{2} first sample a counterfactual reference $x^r$ with the different class label to $x$. We name $x^r$ reference instead of baseline since it is not the integration endpoint of IG\textsuperscript{2}.\footnote{Counterfactual sample $x^r$ directly serves as the baseline in Expected IG~\cite{EG}.}

\begin{definition}
  \emph{(IG\textsuperscript{2} attribution)} Let $x^r$ denote reference, $\gamma^{G}$ denote GradPath, $f$ and $\tilde{f}$ denote the prediction and representation layer of the explained model, IG\textsuperscript{2} attribution for the $i^{th}$ feature of explicand $x$ is defined as:
  \begin{align}
    \phi^{IG^2}_i = \sum_{j=0}^{k-1}  \underbrace{\frac{\partial f(\gamma^{G}(\frac{j}{k}))}{\partial x_i}}_{\text{\emph{explicand's }}} \times \underbrace{ \frac{\partial \Vert \tilde{f}(\gamma^{G}(\frac{j}{k})) - \tilde{f}(x^r) \Vert_2^2 }{ \partial x_i}\times \frac{\eta}{W_j}}_{\text{\emph{counterfactual} }},
    \label{IG2}
  \end{align}
  where $W_j$ is the normalization coefficient and $\eta$ is the step size hyperparameter.
\end{definition}

\par Eq. \ref{IG2} explicitly reveals the nature of IG\textsuperscript{2}: the multiplication of two gradients. Compared to Vanilla IG in Eq. \ref{Riemann}, the major distinction is the Riemann summation weight of explicand's gradients: the weight of counterfactual gradient can highlight the critical features, whereas that in Vanilla IG is a constant. Intuitively, IG\textsuperscript{2} introduces the information of \emph{model representation difference} rather than \emph{naive difference in the input feature space} (i.e., $x_i-x_i'$).

\par Overall, IG\textsuperscript{2} is built on two stages: building GradPath and integrating gradients on GradPath. The following sections respectively introduce them, deriving IG\textsuperscript{2} in Eq. \ref{IG2}.
\renewcommand{\algorithmicrequire}{\textbf{Input:}}
\renewcommand{\algorithmicensure}{\textbf{Output:}}
\begin{algorithm}[t]
  \caption{Compute GradPath and GradCF}\label{algo1}
  \begin{algorithmic}[1]
    \REQUIRE representation layer: $\tilde{f}$

    explicand: $x$

    reference: $x^r$

    step size: $\eta$

    steps: $k$

    \text{ }
    \STATE $ \delta \Leftarrow \vec{0}\ \lhd $ initiate perturbation
    \STATE $\gamma^{G}(1) = x\ \lhd$  initiate GradPath \footnotemark
    \FOR{$j= k-1$ to $0$}
    \STATE $ g = \frac{\partial \Vert \tilde{f}(x+\delta) - \tilde{f}(x^r) \Vert_2 }{ \partial x}$
    \STATE $ W = \Vert g \Vert_{1or2}\ \lhd$  $\ell_1$ or $\ell_2$ norm
    \STATE $ g = \eta\cdot \frac{g}{W} \lhd$  normalized iterative gradient
    \STATE $ \delta = \delta - g \lhd$ update total perturbation
    \STATE $ \gamma^{G}(j/k) = x + \delta\ \lhd$  store the CF at each step
    \ENDFOR

    \text{ }

    \ENSURE GradCF: $\gamma^{G}(0)$,\quad GradPath: $\gamma^{G} $

  \end{algorithmic}
\end{algorithm}


\subsection{Building GradPath}
\par Fig. \ref{framework} illustrates how GradPath is built during the iterative search of GradCF. The motivation of GradCF is to provide a counterfactual explanation: 
\par \emph{Given limited perturbation resource, perturbing which features on explicand $x$ can make the model consider the perturbed explicand to be most similar to the (counterfactual) reference $x^r$? } 
\footnotetext{GradPath is built on the opposite direction to the conventional path integration. To match path methods, we set GradPath $\gamma^{G}(j)$ from $j=1$ to $j=0$, so that the optimization iteration starts at the point $\gamma^{G}(1)$ (explicand) and ends at the point $\gamma^{G}(0)$ (baseline, GradCF).}

\par This similarity can be measured by the distance between two model representations\footnote{We use activations in the penultimate layer as the representation. The choice of representation layer is discussed in Appendix \ref{append_rep}.} and the perturbation search can be converted to a minimization problem. Denoting the network representation by $\tilde{f}$, the perturbation by $\delta$ and Euclidean distance measure by $\Vert\cdot \Vert^2_2$, the optimization objective for GradCF is:
\begin{align}
   & \min_{\delta} \Vert \tilde{f}(x+\delta) - \tilde{f}(x^r)\Vert^2_2.
  \label{obj}
\end{align}

\par We iteratively solve Eq. \ref{obj} using gradient descent with normalization at each step. GradPath $\gamma^{G}$ is built during the iteration by the trajectory of normalized gradient descent, and the endpoint $x+\delta$ is the target counterfactual baseline, GradCF.

\par  Algorithm \ref{algo1} provides the pseudo-code for computing GradCF and GradPath. GradPath can be defined as:

\begin{definition}
  \emph{(GradPath)} Given reference $x^r$ and model representation $\tilde{f}$, GradPath is defined by a discrete function $\gamma^{G}$, on a feasible set $\{0,\frac{1}{k},\cdots,\frac{k-1}{k},1\}$, for $ 0 \le j \le k-1,\ j\in \mathbb{N}$:

  \begin{align}
    \gamma^{G}(\frac{j}{k}) & _{} = \gamma^{G}(\frac{j+1}{k}) - \frac{\partial \Vert \tilde{f}(\gamma^{G}(\frac{j+1}{k})) - \tilde{f}(x^r) \Vert_2 }{ \partial \gamma^{G}(\frac{j+1}{k})}\frac{\eta}{W_j}, \notag \\
    \gamma^{G}(1)           & = x,
    \label{gradpath}
  \end{align}
  where $W_j$ is introduced in Line 5, Algorithm \ref{algo1}.
\end{definition}

\par \textbf{GradCF as explanation (GradCFE)} We can utilize the difference between the GradCF and explicand to provide a counterfactual feature attribution:
\begin{align}
  \mathrm{GradCFE} = x - \gamma^{G}(0).
  \label{CFE}
\end{align}

\subsection{Integrating gradients on GradPath}
\par IG\textsuperscript{2} integrates feature gradients in the same approach as path methods (in Eq. \ref{PathIG}). Since the GradPath $\gamma^{G}$ is a discrete function, which does not have continuous gradients $\frac{\partial \gamma(\alpha)}{\partial \alpha}$, we can operate finite difference to approximate the path gradient. Both forward difference and backward difference are feasible. According to Eq. \ref{gradpath}, the path gradient of GradPath w.r.t $\alpha=\frac{j}{k}$ can be computed by both two differences:
\begin{align}
  \label{difference}
  \frac{\partial \gamma^G(\frac{j}{k})}{\partial \frac{j}{k}} = \begin{cases}
                                                                  \frac{\gamma^G(\frac{j+1}{k}) - \gamma^G(\frac{j}{k})}{1/k} =  k\frac{\partial \Vert \tilde{f}(\gamma^{G}(\frac{j+1}{k})) - \tilde{f}(x^r) \Vert_2 }{ \partial \gamma^{G}(\frac{j+1}{k})}\frac{\eta}{W_j} \\
                                                                  \frac{\gamma^G(\frac{j}{k}) - \gamma^G(\frac{j-1}{k})}{1/k}  = k\frac{\partial \Vert \tilde{f}(\gamma^{G}(\frac{j}{k})) - \tilde{f}(x^r) \Vert_2 }{ \partial \gamma^{G}(\frac{j}{k})}\frac{\eta}{W_j}     \\
                                                                \end{cases},
\end{align}
where the first line is the forward difference and the second is the backward difference.
\par Using the backward difference in Eq. \ref{difference} and Riemann sum of Eq. \ref{PathIG}, we derive IG\textsuperscript{2} attribution (for the $i^{th}$ feature):
\begin{align}
  \phi^{IG^2}_i & = \int_{0}^{1} \frac{\partial f(\gamma^G(\alpha))}{\partial \gamma^G(\alpha)_i}  \frac{\partial \gamma^G(\alpha)_i}{\partial \alpha} d\alpha \notag                                                                                            \\
                & \approx \sum_{j=0}^{k-1} \frac{\partial f(\gamma^{G}(\frac{j}{k}))}{\partial \gamma^{G}(\frac{j}{k})_i} \times \frac{ \partial \Vert \tilde{f}(\gamma^{G}(\frac{j}{k})) - \tilde{f}(x^r) \Vert_2 }{ \partial \gamma^{G}(\frac{j}{k})_i} \notag \\
                & \quad \times \frac{k\eta}{W_j} \times (\frac{j+1-j}{k}) \notag                                                                                                                                                                                 \\
                & =\sum_{j=0}^{k-1} \frac{\partial f(\gamma^{G}(\frac{j}{k}))}{\partial x_i} \times \frac{ \partial \Vert \tilde{f}(\gamma^{G}(\frac{j}{k})) - \tilde{f}(x^r) \Vert_2 }{ \partial x_i} \times \frac{\eta}{W_j},
\end{align}
where the two denominators $\partial \gamma^{G}(\frac{j}{k})_i$ are substituted as $\partial x_i$, due to $\gamma^{G}(\frac{j}{k}) = x + \delta$ (Line 8, Algorithm \ref{algo1}).

\par Practically, the integration direction (from baseline to explicand) is opposite to the GradPath search direction (from explicand to baseline). Thus, backward difference requires the complete GradPath to be computed first, and then the gradient is integrated. If using forward difference, the gradient integral can be computed simultaneously with the GradPath.

\par For the implementation details of IG\textsuperscript{2}, Section \ref{imple_detail} comprehensively discusses the hyperparameter impacts on IG\textsuperscript{2} attribution, including reference, step size, step number, normalization, and similarity measures. The IG\textsuperscript{2} computational cost is also analyzed and compared with other attribution methods.


\begin{figure}[t]
  \setlength{\belowcaptionskip}{-0.3cm}
  \centering
  \includegraphics[width=0.42\textwidth]{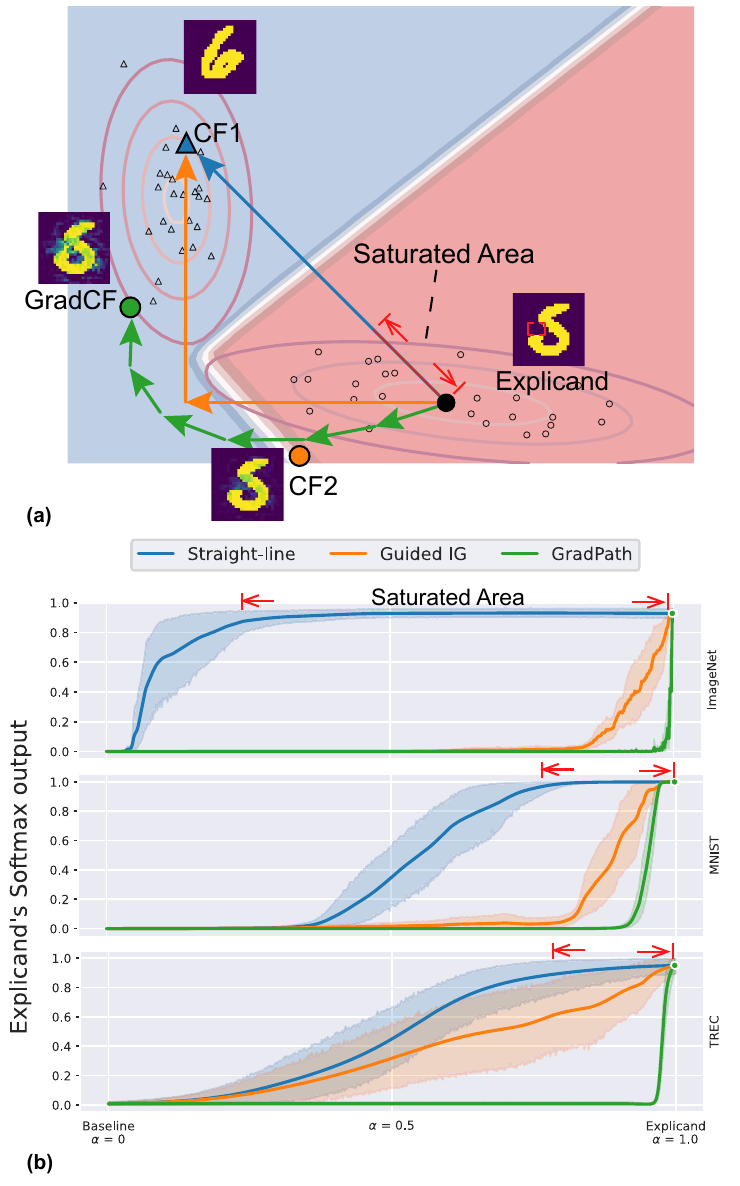}
  \caption{
    \textbf{(a) Illustration for three counterfactual examples and three integration paths:} CF1 sampled from counterfactual data distribution, CF2 generated by an adversarial attack and the GradCF using CF1 as the reference. The saturated area on straight-line path is marked in red. An MNIST explicand (digital 5) and three CFs (digital 6) are plotted.  \textbf{(b) Graphs of explicand's Softmax prediction} along integration paths on ImageNet, MNIST, and TREC, averaged on 100 samples of each dataset.}
  \label{2gau_example}
\end{figure}
\begin{figure}[!t]
  \centering
  \includegraphics[width=0.35\textwidth]{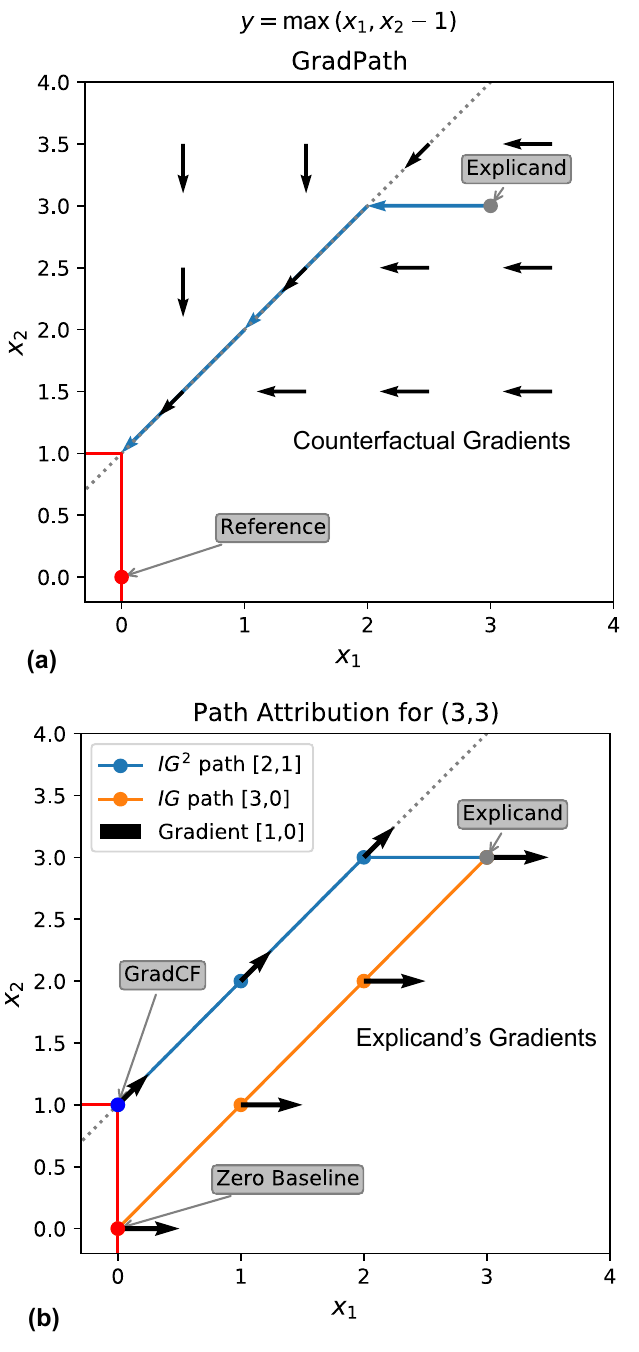}
  \caption{\textbf{The toy example for feature attribution} with function $y=\max(x_1,x_2-1)$ on explicand $(x_1,x_2)=(3,3)$ with the zero reference $(0,0)$. \textbf{(a)} GradPath on counterfactual gradients in blue line. \textbf{(b)} The attribution results of three methods in the legend.}
  \label{toy_example}
\end{figure}
\section{Interpreting IG\textsuperscript{2}}
\label{Interpreting}
\par The novel baseline (GradCF) and novel integration path (GradPath) are two major contributions of IG\textsuperscript{2}, so we discuss IG\textsuperscript{2} by interpreting the superiority of these two components. Specifically, Section \ref{interp_gradcf} and Section \ref{inter_gp} answer the questions: \emph{Why GradCF and GradPath are better baseline and integration path than the existing methods?}
\par Theoretically, the desirable axioms of IG\textsuperscript{2} and GradCFE are justified in Section \ref{Axioms}.


\subsection{GradPath: mitigating saturation effects}
\label{inter_gp}
\par The superiority of GradPath for feature attribution is discussed from the perspective of saturation effects.
\begin{definition}
  \emph{(Saturation effects)}~\cite{guided,Saturation} The straight-line path of IG is susceptible to travel through the saturated area where model output is not changing substantially with respect to $\alpha$. In this area, the feature gradient is not pointing toward the integration path. This phenomenon leads to the accumulation of noisy attributions, called saturation effects.
  \label{def_sat}
\end{definition}

\par The integral value in Eq. \ref{PathIG} can be decomposed into two multipliers: (input) feature gradient and path direction. The presence of a saturated area indicates that the feature gradient and integration path are in dissimilar directions (otherwise, the model prediction should drop quickly), which means the path is not moving on the important features. This will result in incorrect feature attributions, and a good integration path should avoid this undesirable area.

\par The saturation effects were analyzed in works~\cite{guided,Saturation}. Some techniques have been proposed for this issue, for example, averaging over multiple straight paths~\cite{XRAI,SmoothGrad} and splitting the straight paths into different segments~\cite{Saturation}. Guided IG~\cite{guided} explicitly avoids the saturated area by designing the path based on the absolute values of feature gradients. However, the path of Guided IG is still constrained in the hyper-rectangular with the straight-line path as diagonal.

\par GradPath effectively mitigates the saturation effects. Recalling the objective of GradCF in Eq. \ref{obj}, we minimize the distance to the counterfactual model representation, which implicitly means the explicand's prediction is meanwhile minimized. Each step of GradPath points toward the steep direction that rapidly decreases the model prediction (shown in Fig. \ref{2gau_example}).

\par As a supplementary for graphs in Fig. \ref{Paths_comparison}, Fig. \ref{2gau_example}b displays the average output curves on three datasets. Compared with straight line, Guided IG's path and GradPath both avoid the saturated area. Since Guided IG's path is restricted, GradPath can get out of the saturated area more quickly.

\par \textbf{Toy example:} Fig. \ref{toy_example} displays a toy example, showing GradPath contributes to better feature attribution. Because $x_2$ is always smaller than $x_1$ on the straight-line path, the $\max$ function signs zero gradients on the smaller feature (right black arrows). This causes IG and gradient methods both assign zero attribution on feature $x_2$, which is inconsistent with intuition. Contrarily, integrated gradients on GradPath give nonzero attribution on $x_2$ with a more reasonable baseline, GradCF $(0,1)$.

\begin{figure*}[!t]
  \centering
  \includegraphics[width=0.82\textwidth]{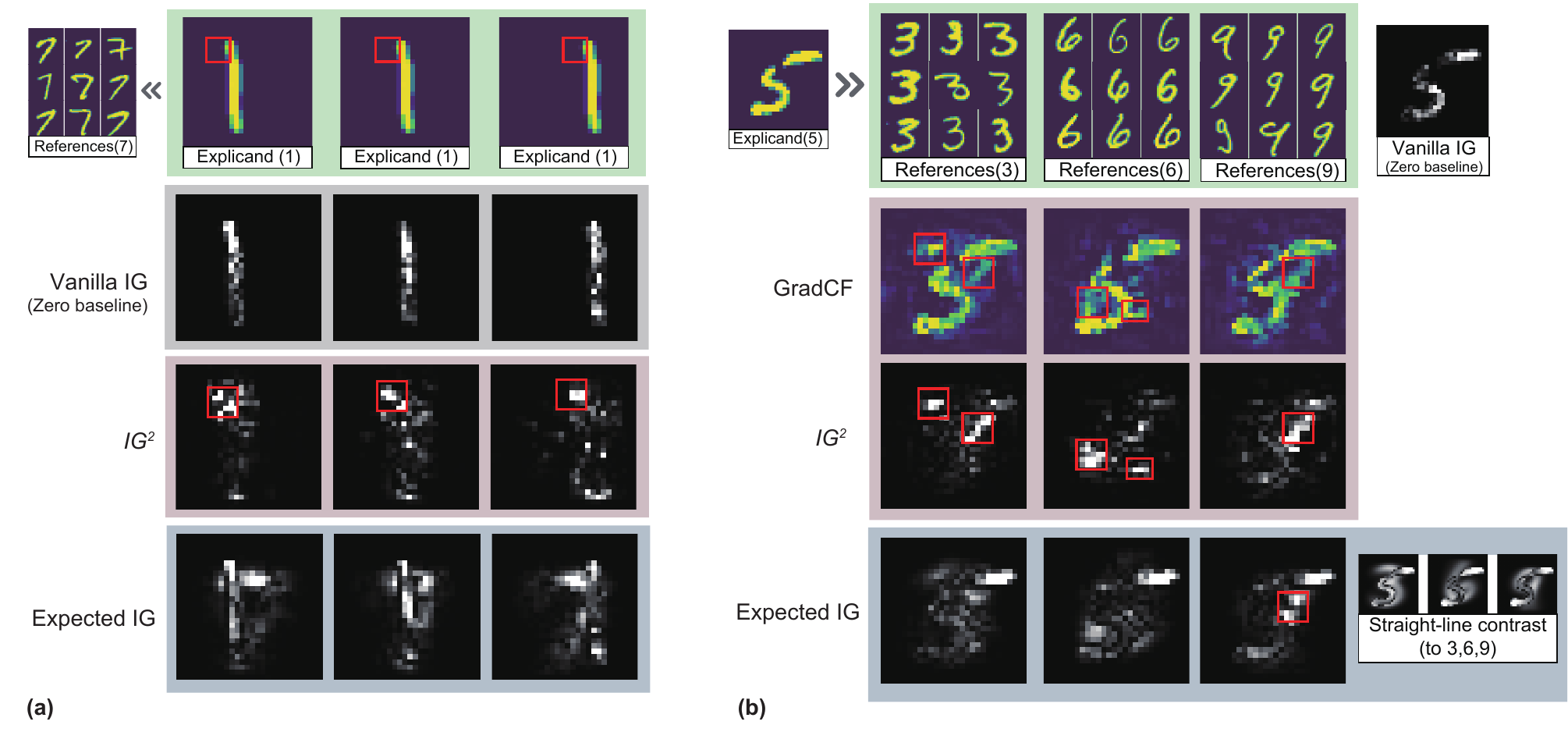}
  \caption{
    \textbf{Feature attributions of MNIST samples.} \textbf{(a) Shifted digital 1s explained with the references of digitals 7.} \textbf{(b) Digitals 5 explained with the references of different categories (digitals 3, 6, 9).} The \emph{most} critical areas that distinguish the explicand to reference are marked by red boxes, e.g., the explained digitals 1 will become digital 7 if we filled these areas. IG\textsuperscript{2} is compared with Vanilla IG (using black baseline) and Expected IG (using references as the baselines).  }
  \label{mnist_example}
\end{figure*}

\subsection{GradCF}
\label{interp_gradcf}

\par Based on works in the field of counterfactual explanation (CFE), we summarize a good counterfactual baseline in path methods should have the following desirable properties:
\begin{itemize}
  \item \emph{Validity}~\cite{CF_GDPR} The counterfactual baseline should be classified in the desired class (different from the explicand).
  \item \emph{Data manifold closeness}~\cite{CFE_review} It would be hard to trust a counterfactual if it resulted in a combination of features that were utterly unlike any observations the classifier has seen before. Therefore, a generated counterfactual should be realistic in the sense that it is near the training data.
  \item \emph{Explicand relevance} A good counterfactual example should be related to the explicand to directly contrast features. Though almost all the CFE methods~\cite{CFE_review} generate counterfactual examples based on the explicand, the baselines of most existing path methods are explicand-agnostic (see Table \ref{Summary}).
\end{itemize}

\par The commonly used uninformative baselines (black and white) violate all three properties, which are unrealistic and do not have any information on classes. For instance, the black pixel will not be attributed with an all-black baseline, even though it contributes to the model output.
\par Expected IG~\cite{EG} solves this by using the samples from the data distribution (CF1 in Fig. \ref{2gau_example}), which satisfies the validity and data manifold closeness. However, the sampling procedure of train data is \emph{explicand-agnostic}, which provides inaccurate contrast in the feature space (discussed in Fig. \ref{mnist_example}a).
\par A basic counterfactual example (CF2 in Fig. \ref{2gau_example}) is also compared. CF2 is generated with projected gradient descent (PGD)~\cite{madry2018towards} attack for the minimal perturbation that causes the model to give a counterfactual prediction~\cite{ignatiev2019relating}. Though this basic CF is related to the explicand, it is usually unrealistic and violates the \emph{Data manifold closeness}.
\par Distinctively, GradCF satisfies all three describable properties, which correlates with the explicand and stays on the manifold of counterfactual data. In other words, the generation of GradCF simultaneously implies the closeness to feature manifolds of explicand and counterfactual.

\par \textbf{MNIST examples:} Fig. \ref{mnist_example}a demonstrates the significance of \emph{Explicand relevance}. Using digitals 7 as the references, we explain shifted digitals 1 at different positions. To human intuition, the critical areas distinguishing digital 1 to 7 are at the top left of digital 1 (see the red boxes). Fig. \ref{mnist_example}a shows that the highlighted areas of IG\textsuperscript{2} are synchronized with the shift of explicands and critical areas, that is \emph{Explicand relevance}.
\par As for Expected IG, using reference samples as baselines only provides the naive pixel contrast at the input feature space. This will result in explanations that are irrelevant to the explicand, which is obviously inconsistent with the intuition (see the last rows in Fig. \ref{mnist_example}).
\par Fig. \ref{mnist_example}b shows the impact of the different references, where the critical areas are marked by red boxes. Attributed to the gradients of counterfactual classes, the GradCF and IG\textsuperscript{2} significantly highlight the critical areas of all three references. Due to the same issue of straight paths, attributions of Expected IG are overly focused on the upper right corners of digital 5, which is not the \emph{most} critical areas\footnote{The ablation study is conducted in Appendix Section \ref{append_MNIST} for validating the most critical areas that distinguishes digital 5 to digitals 3, 6, 9.}. As for IG with the all-black baseline, any pixels out of the explained digitals will not have the attributions, which is incomplete for the explanation.
\par In Section \ref{imple_ref}, we will further demonstrate the effects of counterfactual references on IG\textsuperscript{2} with multiple datasets.

\subsection{Axioms of IG\textsuperscript{2}}
\label{Axioms}
\par Works~\cite{sundararajan2020many,IG,SHAP} claimed that path methods are the unique methods that satisfy certain desirable axioms. As a subset of path methods, we justify IG\textsuperscript{2} satisfies the following four axioms: \emph{Completeness}, \emph{Dummy}, \emph{Implementation Invariance} and \emph{Symmetry}, and GradCFE also satisfies the latter three.

\par \begin{definition}
  \emph{(Completeness)} For every explicand $x$, and baseline $x'$, the attributions $\phi_i$ add up to the prediction difference $f(x)-f(x')$:
  \begin{align}
    \phi_i = f(x)-f(x').
  \end{align}
\end{definition}
\begin{remark}
  Like other path integral methods, IG\textsuperscript{2} also integrates the gradient in a conservative (input) vector field. Since all the path methods satisfy Completeness regardless of the path shape (see \cite{IG,friedman2004paths}), IG\textsuperscript{2} satisfies Completeness.
\end{remark}

\par Due to our proposal utilizing the model representation, we extend the definitions of the three remaining axioms to the representation layer version, the extension of which is respectively (i.e., resp.) shown in the brackets. Notably, if we use the output layer as representation, no extension is needed.
\par \begin{definition}
  \emph{(Dummy)} Dummy features get zero attributions. A feature $i$ is dummy in a function $f$ (\emph{resp. $\tilde{f}$}) if for any two values $x_i$ and $x'_i$ and every value $x_{N\setminus i}$ of the other features, $f(x_i; x_{N\setminus i}) = f(x'_i; x_{N\setminus i})$ (\emph{resp. $\tilde{f}(x_i; x_{N\setminus i}) = \tilde{f}(x'_i; x_{N\setminus i})$}). Conceptually, the feature that is not referenced by the model naturally requires zero attributions.
\end{definition}
\begin{remark} IG\textsuperscript{2} satisfies Dummy and GradCFE satisfies Dummy at the representation layer. The latter is a sufficient condition for the former.
\end{remark}

\begin{proof}
  According to Eq. \ref{gradpath}, given any point $\gamma^G(\alpha)$ on GradPath, the change on feature $\gamma^G(\alpha)_i$ is proportional the normalized gradient:
  \begin{align}
    \Delta \gamma^G(\alpha)_i & \propto \frac{\partial \Vert \tilde{f}(\gamma^G(\alpha)) - \tilde{f}(x^r) \Vert_2^2 }{ \partial \gamma^G(\alpha)_i}\times \frac{1}{W} \\
                              & =  \lim_{h\to 0}  \Big( \partial \Vert \tilde{f}(\cdots,x_i+h,\cdots) - \tilde{f}(x^r) \Vert_2^2 \notag                               \\
                              & \quad - \Vert \tilde{f}(\cdots,x_i,\cdots) - \tilde{f}(x^r) \Vert_2^2 \Big) \times \frac{1}{hW}.
  \end{align}
  \par Based on the definition of Dummy, $ \tilde{f}(\cdots,x_i+h,\cdots)\equiv\tilde{f}(\cdots,x_i,\cdots) $, so that
  \begin{align}
    \Delta \gamma^G(\alpha)_i \equiv 0.
  \end{align}
  \par According to the definitions in Eq.\ref{IG2} and Eq. \ref{CFE}, we get the zero attributions of both IG\textsuperscript{2} and GradCFE for dummy features at the representation layer.
  \par The Dummy axiom of IG\textsuperscript{2} does not require the feature to be dummy at the representation layer. Suppose the feature is dummy only for the model output (unextended definition). In that case, IG\textsuperscript{2} still assigns zero attribution on this feature, the proof of which is similar (explicand's gradient for dummy features in Eq. \ref{IG2} constantly equals to zero) and not repeated here.
\end{proof}

\begin{definition}
  \emph{(Implementation Invariance)} Two networks are functionally equivalent if their outputs (\emph{resp. representations}) are equal for all inputs, despite having very different implementations. Attribution methods should satisfy Implementation Invariance, i.e., the attributions are always identical for two functionally equivalent networks and do not refer to implementation details.
\end{definition}
\begin{remark} IG\textsuperscript{2} and GradCFE satisfy Implementation Invariance. The former is a sufficient condition for the latter. IG\textsuperscript{2} and GradCFE only relies on the function gradients of output and representation (GradCFE only concerns representation gradient), which are invariant to the models' internal implementation before the representation layer.
\end{remark}

\begin{definition}
  \emph{(Symmetry)} For every function $f$ (\emph{resp. $\tilde{f}$}) is symmetric in two variables $i$ and $j$, if $f(\cdots,x_i,x_j,\cdots)=f(\cdots,x_j,x_i,\cdots)$ (\emph{resp. $\tilde{f}$}). If the explicand $x$ are such that $x_i = x_j$, the attributions of symmetric function for features $i$ and $j$ should be equal. Conceptually, under the symmetric function, the identical symmetric variables receive identical attributions.
\end{definition}
\begin{remark} IG\textsuperscript{2} and GradCFE preserve Symmetry. The former is a sufficient condition for the latter. Notably, previous path methods are symmetry requires the variable $i$ and $j$ of the baseline $x'$ are also identical, $x'_i=x'_j$, while our methods do not. The explicand with identical symmetric variables will consequently lead to the identical symmetric variables in the synthesized baseline GradCF.
\end{remark}
\begin{proposition}
  (Guided IG~\cite{guided}) If the values of symmetric variables are equal at every point of the integration path, then their attributions are equal. Therefore, such a path attribution method is symmetry preserving.
  \label{symmetry}
\end{proposition}
\begin{proof}
  According to Proposition \ref{symmetry}, IG\textsuperscript{2} and GradCFE preserve Symmetry if the values of symmetric features on GradPath $\gamma^G$ are equal at every point, the proof of which is in the following.
  \par Given explicand $x$, we only focus on the symmetric variables $i$ and $j$, $x=(\cdots,x_i,x_j,\cdots)$, where $x_i=x_j$. According to Eq. \ref{gradpath}, the change on variable
  $x_i$ along GradPath $\gamma^G$ is:
  \begin{align}
    \Delta x_i & =  - \frac{\partial \Vert \tilde{f}(\gamma^{G}(\alpha)) - \tilde{f}(x^r) \Vert_2^2 }{ \partial x_{i}}\times \frac{\eta}{W}                       \\
               & = (\tilde{f}(\gamma^{G}(\alpha)) - \tilde{f}(x^r))\times \frac{-2\eta}{W} \times \frac{\partial \tilde{f}(\gamma^{G}(\alpha)) }{ \partial x_{i}} \\
               & = \lim_{h\to 0} \frac{ \tilde{f}(\cdots,x_i+h,x_j,\cdots ) - \tilde{f}(\cdots,x_i,x_j,\cdots )}{ \partial h} \notag                              \\
               & \quad \times \frac{-2\eta}{W} \times (\tilde{f}(\gamma^{G}(\alpha)) - \tilde{f}(x^r)).
    \label{sysgrad_i}
  \end{align}
  \par and similarly, we can get the change on variable $x_j$:
  \begin{align}
    \Delta x_j & = \lim_{h\to 0} \frac{ \tilde{f}(\cdots,x_i,x_j+h,\cdots ) - \tilde{f}(\cdots,x_i,x_j,\cdots )}{ \partial h} \notag \\
               & \quad \times \frac{-2\eta}{W} \times (\tilde{f}(\gamma^{G}(\alpha)) - \tilde{f}(x^r)).
  \end{align}
  \par Because $\tilde{f}$ is symmetric for variable $i$ and $j$:
  \begin{align}
    \tilde{f}(\cdots,x_i+h,x_j,\cdots ) =  \tilde{f}(\cdots,x_i,x_j+h,\cdots ),
    \label{sysgrad_j}
  \end{align}
 two gradients in Eq. \ref{sysgrad_i} and Eq. \ref{sysgrad_j} are equal and we get the identical changes on variable $i$ and $j$:
  \begin{align}
    \Delta x_i = \Delta x_j.
  \end{align}
  \par Starting from the identical value $x_i=x_j$ of the explicand, we can get the symmetric variables $x_i$ and $x_j$ are equal at every step of GradPath. Thus, both IG\textsuperscript{2} and GradCFE provide identical attribution for the symmetric variables.

\end{proof}

\section{Related works}
\label{relat}
This section first introduces the previous attribution methods in the field of XAI, and then systemically contrasts our proposal with the related works in three sub-fields: path attribution, Shapley values and adversarially counterfactual explanation.

\subsection{Feature attribution}
\par Feature attribution is a post-hoc method that explains AI models by scoring the feature contributions to the model output~\cite{abhishek2022attribution}.
\par Gradient-based methods are widely used for attribution. One of the earliest successful work is DeconvNet~\cite{zeiler2014visualizing}, which applies a ReLU non-linearity to the gradient computation. Based on Vanilla Gradient~\cite{saliencymaps1} and DeconvNet, Guided Backpropagation~\cite{Guided_BP} introduced an additional guidance signal from the higher layers. Furthermore, Class Activation Map (CAM) was developed~\cite{zhou2016learning} and its variants like Grad-CAM~\cite{Gradcam} also achieved success.
\par Another class of attribution methods is based on perturbation, which analyses the model sensitivity by perturbing the input features. Occlusion sensitivity maps~\cite{zeiler2014visualizing} was one early method, which perturbs the input image with grey squares and observes the model prediction. Not restricted to deep models, LIME~\cite{ribeiro2016should} can be applied to any prediction model by training a linear proxy model. 

\par Shapley values can be considered a particular example of perturbation-based methods and is justified to be the unique method that satisfies certain desirable axioms~\cite{shapley1997value}. However, computing the exact Shapley value is NP-hard~\cite{matsui2001np}, which is prohibitive for deep neural networks. Thus, the related works are dedicated to efficiently approximating the Shapley values with fewer model evaluations. Some works are based on sampling: Strumbelj et al.~\cite{strumbelj2010efficient} and Mitchell et al.~\cite{mitchell2022sampling} respectively proposed Monte Carlo and quasi-Monte Carlo for randomly sampling permutations; KernelSHAP~\cite{SHAP} used LIME to reduce the number of samples; Chen et al.~\cite{chen2018shapley} leveraged the underlying graph structure for the structured data; Wang et al.~\cite{wang2022accelerating} took the advantages of contributive cooperator selection; Ancona et al.~\cite{ancona2019explaining} introduced probabilistic deep network to approximately propagate the Shapley value through the network layers. To further accelerate the approximation, FastSHAP~\cite{jethani2021fastshap} trained a surrogate model to fast generate the explanation, which avoids the expensive sampling procedures; DeepSHAP~\cite{SHAP} approximated the Shapley value at each layer and merged them by DeepLIFT~\cite{DeepLIFT} in a backward fashion.


\subsection{Path methods}
\label{relat_pathmethod}
\par Path method is a popular feature attribution that integrates the gradients along a path. Table \ref{Summary} has summarized the existing path methods to the best of our knowledge. IG\textsuperscript{2} is the first path method with both model-specific path and baseline.

\par \textbf{Contrasting Blur IG~\cite{9157016}} Blur IG resembles IG\textsuperscript{2} in using an iterative algorithm to simultaneously construct the baseline and the integration path. The nature of Blur IG is to iteratively build a path from the explicand to the baseline of a totally blurred explicand. Since the blurred baseline fully depends on the explicand, it lacks the counterfactual information within the model and data distribution. Moreover, Blur IG is restricted to images and not applicable to tabular data.

\par \textbf{Contrasting Expected IG~\cite{EG}} Expected IG uses the informative baselines from the data distribution, which is the reference of our proposal. Nevertheless, the baseline is irrelevant to the explicand, and its straight-line path naively contrasts the baseline and explicand by the difference in the feature space, which still suffers from the saturation effects and noise problem. Notably, the Expected IG\textsuperscript{2} (in Section \ref{sec_eig2}) follows the Expected IG's idea to calculate average attributions by sampling references from the data distribution.

\par \textbf{Contrasting Guided IG~\cite{guided}} Guided IG is one close work to IG\textsuperscript{2}. Firstly, our proposal implies a similar motivation of Guided IG (discussed in Section \ref{inter_gp}). Second, both integration paths are iteratively calculated based on the gradient information to explicitly or implicitly avoid the saturated area. We argue that the shape of GradPath is a generalization of Guided IG's path. If using $\ell_1$ normalization in Eq. \ref{l1norm}, the GradPath has the identical shape as Guided IG's. We select the sparse features with the largest counterfactual gradients, while Guided IG selects the smallest explicand's gradients in the converse direction (the identical shape does not guarantee the identical integration path).
\subsection{Shapley values}
\par IG methods are the generalization of Aumann-Shapley value, an extension of Shapley value to the continuous setting, which inherits desirable attribution axioms~\cite{IG}. Compared with the sampling-based Shapley value approximations, IG\textsuperscript{2} mainly advances in two aspects:
\begin{itemize}
  \item \textbf{Scalability:} Though recent algorithms can achieve efficient approximations, they are still prohibitive for the high-dimensional input features (e.g., ImageNet samples). They have to apply the superpixel (group pixels) to reduce the input dimension~\cite{mitchell2022sampling,ancona2019explaining,SHAP}, which impairs the explanation quality. On the large models, IG-based methods are much more efficient than most Shapley value algorithms, the computational time comparison to sampling-based KernelSHAP is reported in Appendix \ref{append_compute_cost_sec}.
  \item \textbf{Implicit zero baseline:} Similar to Vanilla IG, many Shapley value methods need to indicate the absence of features by replacing them with zero value, which implicitly defines a zero baseline. The zero baseline's adverse effects on feature attributions have been discussed in the previous sections.
\end{itemize}

\par \textbf{Contrasting DeepSHAP~\cite{SHAP}} Unlike other sampling-based approximations, DeepSHAP is more related to IG methods. Its core part, DeepLIFT, replaces the gradient at each nonlinear function with its average gradient, which is shown to be most often a good approximation of IG~\cite{ancona2017towards}.

\par In summary, though sharing the same theory fundamental, we argue that IG and Shapley value approximations are on two different tracks: the former is mainly designed for explaining large networks with accessible gradients, and the latter is for accurately approximating the exact Shapley value, which is more suitable for the small black-box models.

\subsection{Adversarially counterfactual explanation}
\par The adversarial learning shares the same optimization objective with counterfactual explanation, and they are tightly related~\cite{ignatiev2019relating}. Some works~\cite{9442299,attack_XAI} utilized adversarial attack to explain the network. GradCF differs from counterfactual explanation and adversarial attack in violating the principle that counterfactual explanation should be the small perturbation on the explicand~\cite{CFE_review}. Hence, we argue that GradCF is neither a canonical counterfactual explanation (GradCF still provides a counterfactual explanation) nor an adversarial attack.
\par From the perspective of methodologies, the iterative gradient descent optimization method in Eq. \ref{obj} follows the adversarial attack method. The optimization with $\ell_2$ normalization is from projected gradient descent (PGD)~\cite{madry2018towards} while $\ell_1$ normalization is from sparse adversarial attack~\cite{sparseattack}. The only difference is that adversarial attack methods clamp the computed instance within the neighborhood of explicand to guarantee imperceptible perturbations.
\par Notably, the same $\ell_2$ normalized optimization method is also used in work~\cite{ilyas2019adversarial}. Nevertheless, they focus on searching the robust features under the adversarial attack while the model explanation is out of their scope.

\begin{table}[t]
  \centering
  \caption{Evaluation of attribution methods on XAI-BENCH}
  \label{metric_XAIBENCH}
  \begin{threeparttable}
  \begin{tabular}{@{}llllll@{}}
    \toprule
                                         & fai.($\uparrow$) & mon.($\uparrow$) & ROAR($\uparrow$)         & G-S($\uparrow$)   & inf.($\downarrow$) \\ \midrule
    Random                               & -0.033        & 0.458          & 0.332          & -0.060        & 0.034        \\
    Vanilla IG                           & 0.349         & 0.440          & 0.356          & 0.714         & 0.015        \\
    Guided IG                            & 0.342         & 0.463          & 0.359          & 0.681         & 0.021        \\
    Expected IG                          & 0.596         & 0.470          & 0.365          & 0.814         & \textbf{0.014} \\
    DeepSHAP                             & 0.380         & \textbf{0.488} & 0.357         & 0.821          & 0.014        \\
    KernelSHAP                           & 0.370         & 0.435         & 0.340          & \textbf{0.901} & 0.015         \\
    \textbf{IG\textsuperscript{2}(Ours)} & \textbf{0.610}& 0.486      & \textbf{0.377} & 0.833         & 0.021          \\

    \bottomrule
  \end{tabular}
  \begin{tablenotes}
    \footnotesize
    \item[*] fai.:faithfulness mon.:monotonicity G-S:GT-Shapley inf.:infidelity
  \end{tablenotes}
\end{threeparttable}
\vspace{-4mm}
\end{table}

\begin{table}[t]
  \renewcommand\arraystretch{1.4}
  \centering
  \caption{Evaluation of attribution methods on real-world datasets}
  \label{metrics}
  \begin{threeparttable}
    \begin{tabular}{@{}llcc|cc@{}}
      \toprule
                                &                                      & \multicolumn{2}{c|}{Ground truth}          & \multicolumn{2}{c}{SIC AUC}                                     \\
      Datasets                  & Explainers                           & AUC $\uparrow$                             & SUM $\uparrow$              & ADD $\uparrow$ & DEL $\downarrow$ \\ \midrule
      \multirow{7}{*}{ImageNet} & Gradient                             & 0.482                                      & 0.336                       & 0.467          & 0.209            \\
                                & Vanilla IG                           & 0.536                                      & 0.327                       & 0.476          & 0.205            \\
                                & Guided IG                            & 0.599                                      & 0.464                       & 0.545          & 0.212            \\
                                & Expected IG                          & 0.666                                      & 0.431                       & 0.557          & 0.116            \\
                                & DeepSHAP                             & 0.694                                      & 0.470                       & 0.543          & 0.206            \\
                                & KernelSHAP                           & 0.747                                      & 0.498                       & 0.561          & 0.274            \\
                                & \textbf{IG\textsuperscript{2}(Ours)} & \textbf{0.805}                             & \textbf{0.516}              & \textbf{0.656} & \textbf{0.115}   \\
      \midrule
      \multirow{7}{*}{TREC}     & Gradient                             & \multicolumn{2}{c|}{\multirow{7}{*}{---*}} & 0.909                       & 0.189                             \\
                                & Vanilla IG                           & \multicolumn{2}{c|}{}                      & 0.937                       & 0.159                             \\
                                & Guided IG                            & \multicolumn{2}{c|}{}                      & 0.938                       & 0.156                             \\
                                & Expected IG                          & \multicolumn{2}{c|}{}                      & 0.940                       & 0.141                             \\
                                & DeepSHAP                             & \multicolumn{2}{c|}{}                      & 0.933                       & 0.170                             \\
                                & KernelSHAP                           & \multicolumn{2}{c|}{}                      & 0.933                       & 0.172                             \\
                                & \textbf{IG\textsuperscript{2}(Ours)} & \multicolumn{2}{c|}{}                      & \textbf{0.942}              & \textbf{0.140}                    \\

      \midrule
      \multirow{7}{*}{\shortstack{Wafer                                                                                                                                               \\map}}  &  Gradient         &  0.570                                   &  0.226                    & 0.661               & 0.216                        \\
                                & Vanilla IG                           & 0.732                                      & 0.342                       & 0.829          & 0.061            \\
                                & Guided IG                            & 0.758                                      & 0.450                       & 0.789          & 0.050            \\
                                & Expected IG                          & 0.850                                      & 0.488                       & 0.883          & 0.038            \\
                                & DeepSHAP                             & \textbf{0.863}                             & 0.528                       & 0.890          & \textbf{0.029}   \\
                                & KernelSHAP                           & 0.683                                      & 0.339                       & 0.707          & 0.042            \\
                                & \textbf{IG\textsuperscript{2}(Ours)} & 0.849                                      & \textbf{0.551}              & \textbf{0.898} & 0.036            \\
            \midrule

      \multirow{7}{*}{CelebA$\dagger$} & Gradient                      & 0.748                                      & 0.211                     & 0.745          & 0.296            \\
                                & Vanilla IG                           & 0.705                                      & 0.175                       & 0.740          & 0.314            \\
                                & Guided IG                            & 0.653                                      & \textbf{0.225}                       & 0.744          & 0.370            \\
                                & Expected IG                          & 0.698                                      & 0.189                       & 0.737          & 0.293            \\
                                & DeepSHAP                             & 0.699                                      & 0.179                        & 0.750          & 0.309           \\
                                & KernelSHAP                           & 0.788                                      & 0.212                       & 0.765          & \textbf{0.197}            \\
                                & \textbf{IG\textsuperscript{2}(Ours)} & \textbf{0.795}                             & 0.224              & \textbf{0.815} & 0.205   \\
    \bottomrule
    \end{tabular}
    \begin{tablenotes}
      \footnotesize
      \item[*] The ground truth of TREC is not available.
      \item[$\dagger$] The ground truth of CelebA face attributes is generated by face parsing model pretrained on CelebAMask-HQ~\cite{CelebAMask-HQ}, which is detailed in Appendix \ref{append_clelba}.

    \end{tablenotes}
  \end{threeparttable}
  \vspace{-1.2em}
\end{table}

\section{Experiments}
\label{exper}
\par We conduct the attribution experiments on one synthetic dataset, XAI-BENCH, and four real-world tasks: image classification on ImageNet, question classification on TREC,  anomaly classification on wafer map failure pattern, and face attribute classification on CelebA. We compare IG\textsuperscript{2} with six methods:
\begin{itemize}
  \item \textbf{(Vanilla) Gradient:} The fundamental feature attribution method based on backpropagation, using the input gradient w.r.t. the model's prediction for generating the saliency maps.
  \item \textbf{Vanilla IG, Guided IG, Expected IG:} Three IG-based methods, which have been discussed in the previous sections.
  \item \textbf{KernelSHAP:} A basic sampling method for approximating Shapley values of black-box models, which is a common baseline for approximation algorithms. Due to the scalability limitation, the superpixel technique is applied when attributing ImageNet and wafer map samples with KernelSHAP.
  \item \textbf{DeepSHAP:} A high-speed Shapley value approximation for deep models based on DeepLIFT.
\end{itemize}

\par For the baselines of compared methods, Vanilla IG, KernelSHAP, and Guided IG all use the black image as the (implicit) baseline, Expected IG and DeepSHAP samples the baselines from the same distribution as the counterfactual references of IG\textsuperscript{2}. 



\subsection{XAI Benchmark}
\par First, we evaluate our proposal on the synthetic datasets and metrics released by XAI-BENCH~\cite{XAI_bench}, a benchmark for feature attribution algorithms. Synthetic datasets allow the efficient computation of the ground truth Shapley values and other metrics, which is intractable on real-world datasets. We briefly introduced the dataset and metrics, the details of which can be found in the XAI-BENCH work~\cite{XAI_bench}.
\subsubsection{Synthetic Dataset}
\par The features are sampled from a multivariate normal distribution $\mathbf{X}\sim \mathcal{N}(\mathbf{\mu},\mathbf{\Sigma})$, where $\mathbf{\mu}$ is the mean vector and $\mathbf{\Sigma}$ is the covariance matrix. The labels are binary (0 and 1) and defined over a \emph{piecewise} distribution with the function $\Psi(x)$. The explained model is a trained three-layer perceptron for the regression task on the synthetic dataset. The specification of the model and synthetic dataset are reported in Appendix \ref{append_datamodel_xai}.

\subsubsection{Metrics}
\par We use five metrics from XAI-BENCH: (1) \emph{faithfulness} computes the Pearson correlation coefficient between the attribution and the approximate marginal contribution for each feature; (2)\emph{monotonicity} computes the fraction of the marginal improvement for feature with attribution order $i$ is greater than the marginal improvement for feature with attribution order $i+1$; (3)\emph{ROAR} is remove-and-retrain, which retrains the model with the features removed and the area-under-the-curve (AUC) of the model is computed; (4) \emph{GT-Shapley} computes the Pearson correlation coefficient of the feature attribution to the ground-truth Shapley values (ground-truth marginal improvement); (5) \emph{Infidelity} is computed by considering the effects of replacing each feature with a noisy baseline conditional expectation.

\subsubsection{Results}
\par Table \ref{metric_XAIBENCH} reports the five metrics for evaluating the feature attributions on XAI-BENCH datasets. The Vanilla IG and Guided IG use zero baselines with different paths, and the Expected IG uses data distribution as the baseline. We also use the randomly generated attribution (Random) as a weak comparison.
\par The results evaluate the feature attributions of the small model and low input dimensions. IG\textsuperscript{2} generally outperforms most other methods on the first three metrics. Unsurprisingly, the Shapley value sampling method (KernelSHAP) achieves the best performance on the GT-Shapley metric. The improvement of Expected IG over other IG methods reveals that using the exception of attribution from the data distribution baseline is effective.

\begin{figure*}[!t]
  \setlength{\belowcaptionskip}{-0.4cm}
  \centering
  \includegraphics[width=0.65\textwidth]{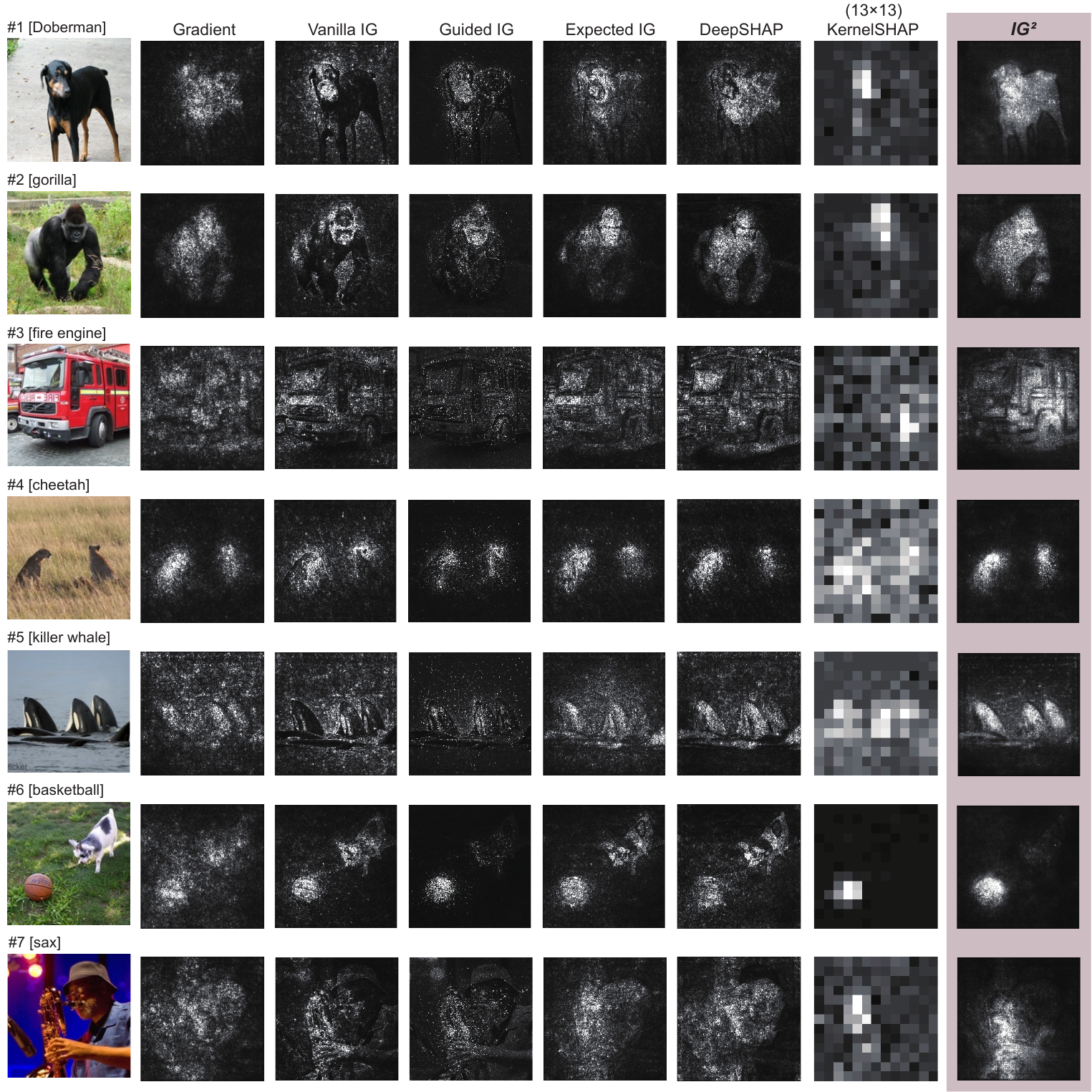}
  \caption{\textbf{Feature attributions on images from ImageNet dataset.} The predicted classes are listed in square brackets. }
  \label{imagenet_explain}
\end{figure*}

\subsection{Metrics for real-world datasets}
We adopt two types of quantitative metrics for evaluating the feature attribution on real-world datasets:
\par \textbf{Ground truth annotation~\cite{8315047}} The first metric requires the ground truth segmentation annotated by humans. The better attribution should be closer to the ground truth annotations. Specifically, this metric treats the attributions as binary classification prediction scores. With changing the threshold of attribution scores to be negative class, the area under the receiver operating characteristic curve is calculated~\cite{guided}, which is called \emph{ground truth-AUC}. We can also use the multiplication of (normalized) attribution scores and ground truth to show the sum of attributions on the annotated features, which is called \emph{ground truth-SUM}.

\par \textbf{Softmax Information Curve (SIC AUC)~\cite{XRAI}} This metric is free from the ground truth annotation and measures how much the attributed features can influence the model prediction, which is similar to the marginal contribution concept in the Shapley theory. The better attribution should have a better focus on where the model is truly looking. There are two metrics in different directions. The first one gradually adds the feature values of the explicand to the background. By sliding an attribution threshold, the feature with the largest attribution is first added and the least the last. The better attribution method should increase the model prediction more quickly, which can be quantized by the area under the Softmax prediction curve w.r.t threshold, \emph{SIC AUC-ADD}. Conversely, another metric deletes the most important feature until all the features are replaced by the background~\cite{EG}. Similarly, the better attribution method should decrease the model prediction quicker, where the AUC w.r.t. threshold is called \emph{SIC AUC-DEL}.

\par Table \ref{metrics} summarizes four metrics of three real-world datasets, where IG\textsuperscript{2} significantly outperforms other methods in general.

\subsection{Image classification explanation}
\label{sec_image_exp}
\subsubsection{Dataset}
\par We validate IG\textsuperscript{2} on a standard image classification dataset, ImageNet. We take the explained images from the ILSVRC~\cite{ILSVRC15} subset (1k classes) with the ground truth annotations. We use the pre-trained classifier of Inception-v3~\cite{Inceptionv3} with input size 299$\times$299. We only consider the images that are predicted as one of the top 5 classes by the Inception-v3 classifier.

\subsubsection{Attributions}
\par Fig. \ref{imagenet_explain} shows ImageNet images explained with four IG-based methods and two Shapley value methods, where IG\textsuperscript{2} generally outperforms previous techniques. Whereas guided IG efficiently reduces the noises by constructing the path that avoids the saturated areas (discussed in Section \ref{inter_gp}), it gives incomplete attributions due to the zero baseline, which is also a drawback of Vanilla IG and KernelSHAP. This problem is especially evident in the images with black subjects (images \#1, \#2, \#4, and \#5 in Fig. \ref{imagenet_explain}).
\par Expected IG and DeepSHAP achieve similar attribution results. They mitigate the incomplete attribution problem by introducing informative baselines providing more attributions on image subjects, but it still suffers from the undesirable noise problem. This is especially obvious in the images with interference objects (e.g., the dog in \#6[basketball] and the player in \#7[sax]). Expected IG and DeepSHAP will incorrectly highlight these objects, which are irrelevant to the explicand class label.

\par IG\textsuperscript{2} combines two advantages of Guided IG and Expected IG, perspectively attributed to two techniques, GradPath and GradCF:
\par \textbf{Less noise by GradPath:} As discussed in Section \ref{inter_gp}, the integration path of IG\textsuperscript{2} successfully mitigates the saturation effects. Image attributions in Fig. \ref{imagenet_explain} also validate this superiority. IG\textsuperscript{2} provides significantly less noisy attributions (less noise on background or irrelevant objects) over IG and Expected IG that use the straight-line path. Compared with Guided IG, IG\textsuperscript{2} is competitive and slightly better on some samples (e.g., image \#7).
\par \textbf{More complete attribution by GradCF:} As discussed in Section \ref{interp_gradcf}, the explicand-specific GradCF of IG\textsuperscript{2} can highlight the critical features that distinguish the explicand from the counterfactual reference. As for images from ImageNet, the critical areas should be the subjects of the explicand label. Appendix Fig. \ref{imagenet_gracfe} shows the difference between the explicand and GradCF (i.e., GradCFE). Based on this counterfactual contrast, IG\textsuperscript{2} attributions highlight the critical features more completely than IG, Guided IG, and even Expected IG.

\par Table \ref{metrics} with quantitative metrics shows our proposal achieves the best performances. Despite the incomplete attributions, the features highlighted by Expected IG are enough to make the classifier give incorrect predictions. This is why Expected IG achieves comparable SIC AUC-DEL value to IG\textsuperscript{2}, but got much worse performances on SIC AUC-ADD and ground truth metrics, which are more dependent on attribution completeness. Since the superpixel technique makes the high attributions concentrated on small areas, KernelSHAP achieves relatively good performance on metrics about ground truth annotation, but its SIC AUC metrics are not desirable.

\subsection{Question classification explanation}
\subsubsection{Dataset}
\par In the field of natural language processing (NLP), question answering is an important task. Question classification can categorize the questions into different semantic classes (whether the question is about location, person or numeric information, etc.), which can impose constraints on potential answers. For instance, the question--``\emph{Where did guinea pigs originate?}'' should be classified as having the answer type [location].
\par We use TREC question dataset~\cite{li2002learning} involving six semantic classes and train a CNN-based classifier (TextCNN)~\cite{textcnn}. We attribute word-level features in order to seek the trigger words that contribute most to the answer type.

\subsubsection{Attributions}
\par Fig. \ref{trec_explain} lists questions sampled from five classes from TREC dataset with word attributions by IG\textsuperscript{2} and IG. IG uses the all-zero embedding vector as the baseline. Compared to IG, the trigger words highlighted by IG\textsuperscript{2} are more consistent with human grammatical perception. We summarize two advantages of IG\textsuperscript{2} over IG.
\par \textbf{Less attributions on weak interrogative words:} Some initial interrogative words are strongly associated with the question types, e.g., ``\emph{where}'' indicates the question for [location] and ``\emph{who}'' indicates [human] (see questions \#2 and \#7 in Fig. \ref{trec_explain}). In this case, these interrogative words should be strongly attributed.
\par On the other hand, some interrogative words are weakly related. For instance, ``\emph{what}'' and ``\emph{which}'' may indicate almost all the question types (questions \#1, \#3, \#4, \#5, \#8, and \#9). Word ``\emph{how}'' itself is ambiguous, which becomes a trigger phrase only when combined with other words (\#6 and \#10). These weakly related interrogative words should be less attributed.
\par Shown in Fig. \ref{trec_explain}, Vanilla IG strongly attributes all the interrogative words, whereas IG\textsuperscript{2} precisely attributes different interrogative words. IG\textsuperscript{2} keeps large attributions on strong interrogative words (``where'' and ``who'' in questions \#2 and \#7), and provides much less attributions on weak interrogative words (the remaining questions).
\par \textbf{More attributions on critical phrases:} Compared with IG, attributions of IG\textsuperscript{2} concentrate more on the critical phrases, such as, ``\emph{drug}'' and ``\emph{name revolt}'' for [entity] (questions \#3 and \#4), ``\emph{how do}'' for \#6[description], ``\emph{sought to}'' for \#7[human], and ``\emph{lengths}'' for \#9[numeric].
\par Table \ref{metrics} validates that word attributions of IG\textsuperscript{2} are more consistent with the model behavior than other methods. Since the number of input features in TREC dataset is small (maximum 37 words), sampling-based KernelSHAP is competitive with other path methods.

\begin{figure}[!t]
  \centering
  \includegraphics[width=0.5\textwidth]{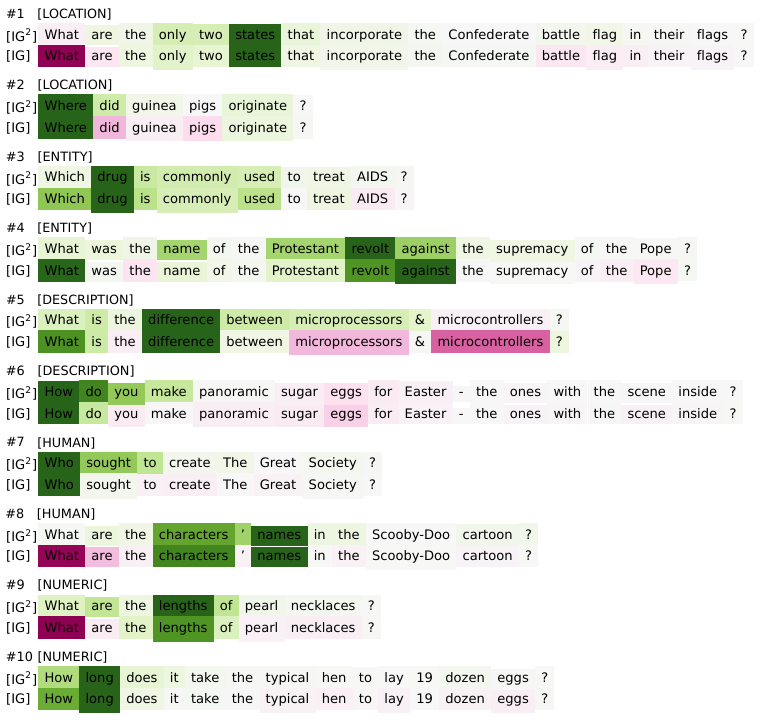}
  \caption{\textbf{Word attributions for questions from TREC dataset.} IG\textsuperscript{2} is compared with Vanilla IG. The color depth indicates attribution strength, and the color type indicates the attribution direction (red is negative, and green is positive). The predicted classes are listed in the top of each question. The complete attributions of all methods are reported in Appendix Fig. \ref{trec_explain_supply}.}
  \label{trec_explain}
\end{figure}

\begin{figure}[!t]
  \setlength{\belowcaptionskip}{-0.4cm}
  \centering
  \includegraphics[width=0.47\textwidth]{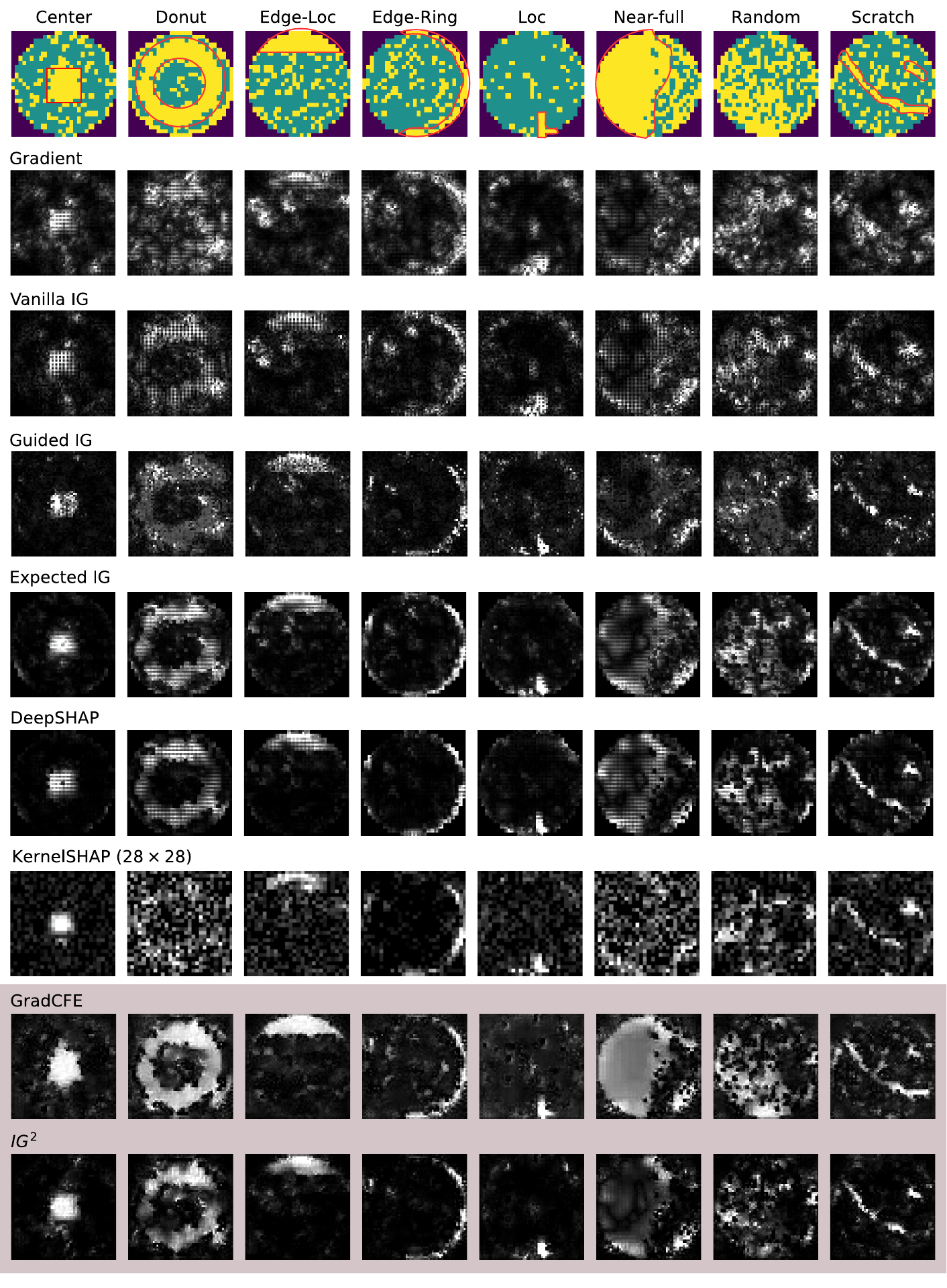}
  \caption{\textbf{Feature attributions on wafer maps with eight different failure patterns.} The first row shows the original anomaly wafer maps with the red ground truth annotation based on expert knowledge.}
  \label{Wafermap_comparison}
\end{figure}

\subsection{Wafer map failure pattern explanation}
\subsubsection{Dataset}
\label{wafermap}
Wafer map analysis is critical in daily semiconductor manufacturing operations. Wafer maps provide visual details that are crucial for identifying the stage of manufacturing at which wafer pattern failure occurs. Instead of manual work, automatically identifying different types of wafer map failure patterns can effectively improve the efficiency of the semiconductor manufacturing process.
\par The explanation of classification deep neural network for wafer map failure pattern can determine which parts (pixels) of the wafer maps are the cause that leads to the failure. This explanation enhances the model's ability to automatically identify the cause of the anomaly wafer maps rather than only recognizing the failure types.
\par Specifically, we use WM-811K dataset~\cite{6932449}, which is the largest known wafer map dataset available to the public. We use a subset of the whole dataset for training the classification model based on convolution neural networks (CNN), achieving a classification accuracy greater than 98.5\% on both train and test sets. The implementation detail of the classification model and sampled instances from WM-811K are provided in Appendix~\ref{append_datamodel_wafer}.

\subsubsection{Attributions}
\par First of all, Fig. \ref{Wafermap_comparison} directly compares the different attributions on eight samples with different patterns in WM-811K dataset. Compared to naive gradient methods, integrated gradients significantly improves feature attribution. Still, the Vanilla IG fails to completely highlight the failure patterns (as the red ground truth) and suffers from the noise problem. Though Guided IG efficiently reduces the noise on the irrelevant features by the designed path, its attribution is still incomplete, caused by the arbitrary baseline containing less counterfactual information. Expected IG and DeepSHAP solve this by using the informative baselines over the data distribution of [nonpattern] instances, but its straight-line path still introduces some noises (especially on the circle edges).
\par Our GradCF solves the inaccurate counterfactual information problem in the existing baselines, shown by GradCFE. It highlights the features contributing to the model representation difference between the explicand and reference. However, it simultaneously accumulates lots of undesirable noises on irrelevant features. IG\textsuperscript{2} successfully solves this side effect by incorporating the gradient of explicand's prediction, which significantly reduces the noise attributions by filtering out the features that have less contribution to the output of explicand's class.
\par Overall, compared with the advanced existing attribution methods, IG\textsuperscript{2} provides less noisy attributions that are most consistent with expert knowledge and human intuition. Moreover, Table \ref{metrics} reports quantitative metrics to evaluate the attributions. The results show that IG\textsuperscript{2} generally outperforms other advanced path attribution methods and KernelSHAP. Our proposal significantly improves the Vanilla IG's performance and is competitive with Expected IG and DeepSHAP.

\begin{figure}[!t]

  \centering
  \includegraphics[width=0.47\textwidth]{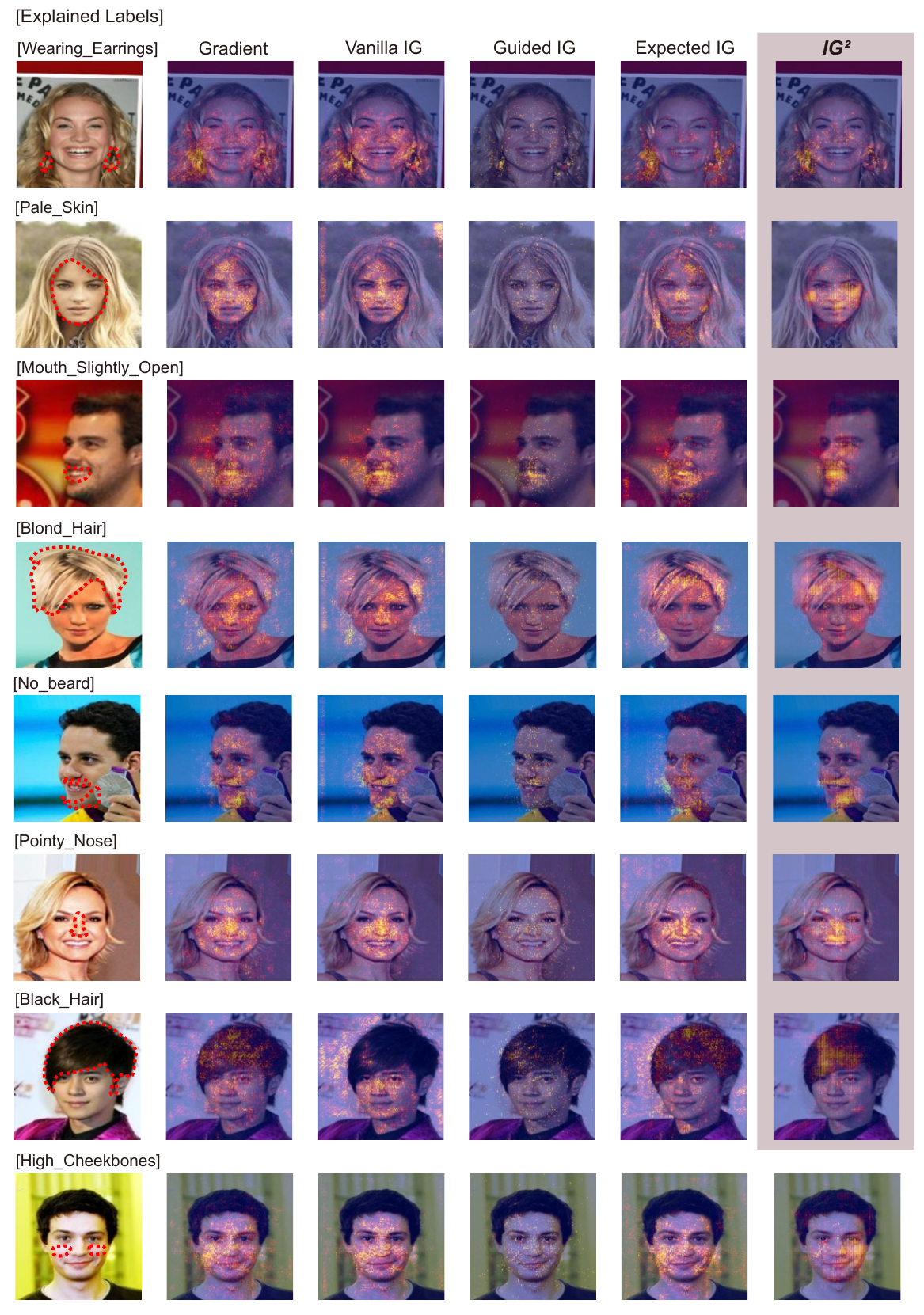}
  \caption{\textbf{Feature attributions on CelebA face attributes.} The explained labels appear at the top of each explicand. The critical facial areas associated with the explained label are marked in red dashed lines. The feature attributions are shown as the heatmaps on the images, where warmer colors (reds and yellows) denote areas of higher importance.}
  \label{celebA_res}
\end{figure}

\subsection{Face attribute classification explanation}
\label{celeba_attr}
\subsubsection{Dataset}
Each face image in CelebA~\cite{liu2015faceattributes} has 40 binary face attribute labels, indicating the presence or absence of specific facial attributes like smiling, wearing earrings, or having a mustache. We train the face attribute classification model on the CelebA dataset, based on MobileNet-v2~\cite{DBLP:journals/corr/abs-1801-04381} with 40 output nodes corresponding to each face attribute. The classification accuracy for each face attributes are reported in Appendix \ref{append_clelba}.
\par For multi-label classifier, we separately explains each output label, i.e., one face attribute at a time~\cite{10.1145/3563039}. We use the counterfactual references that are most relevant to the explicand, i.e., the faces with labels differ in the explained attribute but are closet in other face attributes. The effect of references on CelebA will be analyzed in Section \ref{sec_ref_celebA}.

\subsubsection{Attributions}
\par Fig. \ref{celebA_res} showcases the feature attributions on the CelebA face images. The effectiveness of a feature attribution method is determined by its ability to accurately emphasize the image region associated with the explained label (as indicated by red dashed line areas in Fig. \ref{celebA_res}). Using face image with label [Black\_Hair] as a case study, IG\textsuperscript{2} method demonstrates a more focused attribution towards the hair region, with less noises than other methods. Conversely, Vanilla IG and Guided IG encounter issues with black baselines in face images, mistakenly ignoring black pixels. Although Vanilla Gradient and Expected IG methods show relatively good performance on CelebA datasets, they still suffer from the saturation effect, leading to undesirable attributions on the irrelevant pixels, such as the image background.

\par Overall, the attributions by IG\textsuperscript{2} method are more precisely aligned with the facial regions that are relevant to the explained labels, whereas other methods tend to produce noisier and less accurate attributions on the critical facial regions. 

\subsection{Ablation study}
\label{sec_ablation}
Furthermore, based on the wafer map dataset, we studied the impact of GradCF and GradPath independently as the baseline and integration path for path methods. Table \ref{ablation} reports three different baselines (black, train data, and GradCF) under three different paths (straight-line, Guided IG's, and GradPath). As a baseline under the straight-line path and Guided-IG's path, GradCF achieves the best performance of most metrics compared with the other two. On the other hand, GradPath outperforms than other two paths with the GradCF baseline.
\par In general, we can conclude that: (1) GradCF is a good baseline for path attribution, even independently combined with other paths. (2) GradPath outperforms other paths on the GradCF baseline, which accomplishes IG\textsuperscript{2}.

\begin{table}[t]
  \centering
  \caption{Ablation study of GradCF and GradPath}
  \label{ablation}
  \begin{threeparttable}
    \begin{tabular}{@{}llll|ll@{}}
      \toprule
                                       &                  & \multicolumn{2}{c|}{Ground truth} & \multicolumn{2}{c}{SIC AUC}                                     \\
      Paths                            & Baselines        & AUC $\uparrow$                    & SUM $\uparrow$              & ADD $\uparrow$ & DEL $\downarrow$ \\ \midrule
      \multirow{3}{*}{Straight-line}   & Black            & 0.732                             & 0.342                       & 0.829          & 0.061            \\
                                       & Train data       & 0.810                             & 0.445                       & 0.777          & 0.053            \\
                                       & \textbf{GradCF}  & 0.845                             & 0.510                       & 0.842          & 0.041            \\  \midrule
      \multirow{3}{*}{Guided-IG's}     & Black            & 0.758                             & 0.450                       & 0.789          & 0.050            \\
                                       & Train data       & 0.804                             & 0.493                       & 0.775          & 0.043            \\
                                       & \textbf{GradCF}  & 0.781                             & 0.472                       & 0.819          & 0.037            \\  \midrule
      GradPath (IG\textsuperscript{2}) & \textbf{GradCF}* & 0.849                             & 0.551                       & 0.898          & 0.036            \\ \bottomrule
    \end{tabular}

    \begin{tablenotes}
      \footnotesize
      \item[*] The baseline of GradPath is iteratively synthesized, so only GradCF is available for GradPath.

    \end{tablenotes}
  \end{threeparttable}
\end{table}

\section{Implementation details}
\label{imple_detail}
\par We discuss the details of IG\textsuperscript{2}, including the reference, step, normalization, similarity measure, and computational cost. We provide in-depth analyses of different choices on these hyper-parameters. Some supplementary experiments are reported in Appendix \ref{append_imple}.
\subsection{Reference}
\label{imple_ref}
\par The choice of counterfactual reference is a major hyper-parameter of IG\textsuperscript{2}. Specifically, IG\textsuperscript{2} attributions are sensitive to the \emph{category} of counterfactual reference (while relatively insensitive to different samples of the same category).
\par \textbf{Choice of reference category:} For the classification tasks on different datasets, the ways to sample references are also different:
\begin{itemize}
  \item \textbf{Anomaly classification:} For the dataset consisting of anomaly and normal samples (e.g., wafer map failure patterns), it is natural to use the samples of the normal category as the references for the anomaly explicand rather than other anomaly categories.
  \item \textbf{General classification:} Most datasets do not have such a natural category for reference, such as ImageNet, TREC, etc. Without loss of generality, we randomly sample the references from the different categories. Empirically, we recommend uniformly sampling references from more categories and a few (1 or 2) samples per category.
  \item \textbf{Tricks for denoising:} During experiments, we find using the references of categories that are closely relevant to the explicand can reduce the noises in attributions, but at the cost of losing completeness. This will be beneficial for the explicands with interference terms, such as some ImageNet samples (Section \ref{sec_imagenet_refs}) and multi-label CelebA samples (Section \ref{sec_ref_celebA}).
\end{itemize}

\par Notably, without loss of generality, we take the second one (without denoising tricks) as the default reference choice strategy. The effect of reference category on IG\textsuperscript{2} attributions are discussed on four datasets, MNIST (Fig. \ref{mnist_example}), ImageNet (Section \ref{sec_imagenet_refs}), TREC (Section \ref{sec_trec_ref}), and CelebA (Section \ref{sec_ref_celebA}).

\begin{figure}[!t]
  \setlength{\belowcaptionskip}{-0.4cm}
  \centering
  \includegraphics[width=0.48\textwidth]{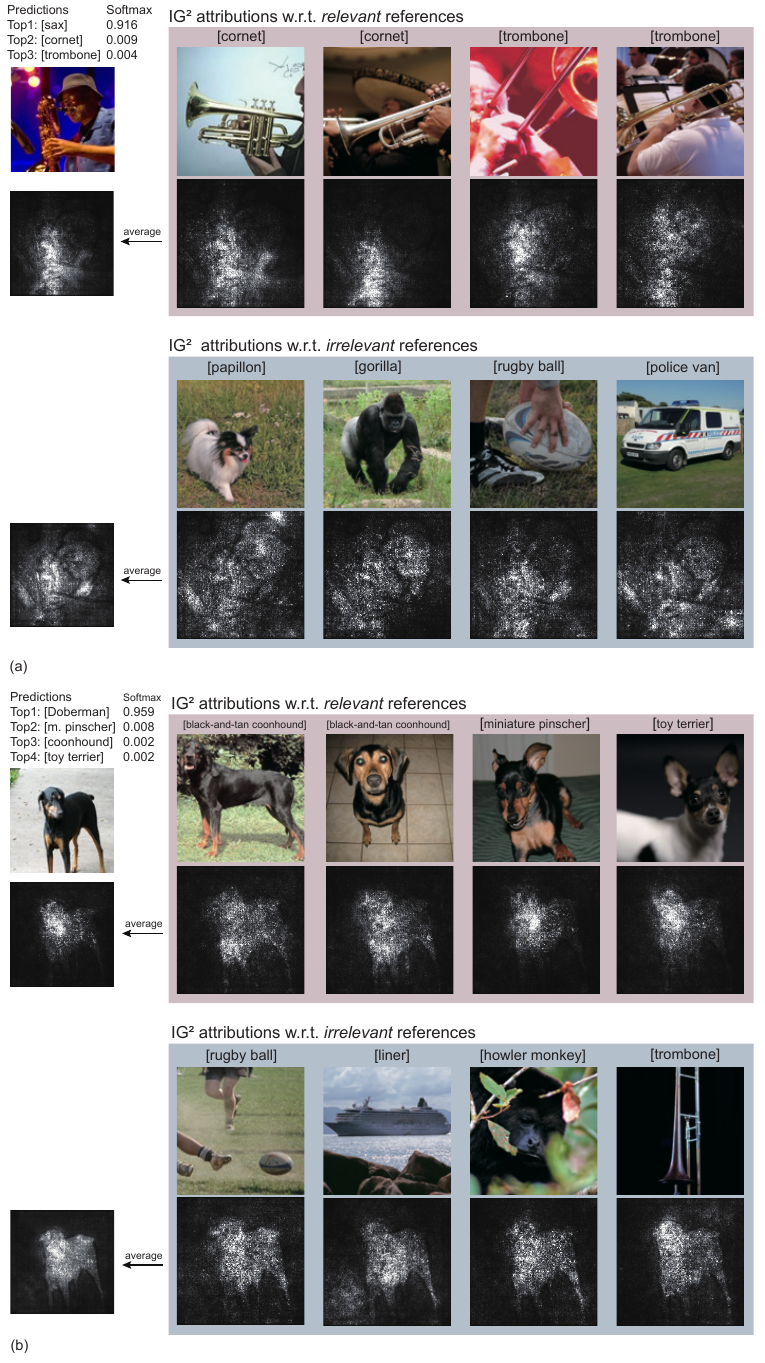}
  \caption{\textbf{Feature attributions w.r.t. different references.} \textbf{(a) Explained image of class [sax]} with interference object ``player''. \textbf{(b) Explained image  of class [Doberman]} on pure background.}
  \label{imagenet_refs}
\end{figure}

\subsubsection{Expected IG\textsuperscript{2} over references}
\label{sec_eig2}
\par To eliminate the influence of the reference choice and reduce noise, we use the expectation of IG\textsuperscript{2} as the final attribution. The Expected IG\textsuperscript{2} is averaged over the references sampled from counterfactual categories, which is defined as:
\begin{align}
  \phi^{IG^2}_i =\mathbb{E}_{x^r\sim D^r} \ \phi^{IG^2}_i (x_r),
  \label{eq_EIG2}
\end{align}
where $D^r$ is the data distribution of the counterfactual categories to the explicand and $\phi^{IG^2}_i (x_r)$ is calculated by Eq. \ref{IG2} as a function of $x_r$. In the following, we make no distinction between IG\textsuperscript{2} and Expected IG\textsuperscript{2}, and use Expected IG\textsuperscript{2} as the practice for experiments.

\begin{table}[]
  \centering
  \caption{Evaluation of IG\textsuperscript{2} reference choice strategies on ImageNet}
  \label{ref_quant}
  \begin{threeparttable}
    \begin{tabular}{@{}llll|ll@{}}
      \toprule
                                &                  & \multicolumn{2}{c|}{Ground truth} & \multicolumn{2}{c}{SIC AUC}                                     \\
      Strategy                  & Set of explicand & AUC $\uparrow$                    & SUM $\uparrow$              & ADD $\uparrow$ & DEL $\downarrow$ \\ \midrule
      \multirow{3}{*}{Relevant} & W/ interference  & 0.905                             & 0.345                       & 0.901          & 0.084            \\
                                & Pure background  & 0.755                             & 0.633                       & 0.537          & 0.119            \\
                                & Whole              & 0.798                             & 0.551                       & 0.641          & 0.109            \\ \midrule
      \multirow{3}{*}{Irrelevant}   & W/ interference  & 0.864                             & 0.273                       & 0.876          & 0.086            \\
                                & Pure background  & 0.781                             & 0.613                       & 0.568          & 0.126            \\
                                & Whole*             & 0.805                             & 0.516                       & 0.656          & 0.115            \\ \bottomrule
    \end{tabular}

    \begin{tablenotes}
      \footnotesize
      \item[*] The strategy reported in Table \ref{metrics}.

    \end{tablenotes}
  \end{threeparttable}
  \vspace{-3mm}
\end{table}

\subsubsection{References of ImageNet}
\label{sec_imagenet_refs}
\par During experiments of ImageNet, we find that the category of reference can influence the IG\textsuperscript{2} attributions. We categorize the references into two types by their labels: \emph{relevant} references and \emph{irrelevant} references. \emph{Relevant} references are the samples of categories that are closely related to the explicand, e.g., the classes in model top 3 predictions, and \emph{irrelevant} references are the samples of other classes.
\par Generally, we find that only using \emph{relevant} references will lead to more concentrated (less noisy) but less complete attributions. Contrarily, the \emph{irrelevant} references can ensure the complete attribution covering the whole image subject but at the cost of introducing more noise. Fig. \ref{imagenet_refs} and Appendix Fig. \ref{append_imagenet_refs} intuitively compares the IG\textsuperscript{2} attributions w.r.t. different references. Moreover, we propose a trick to solve this trade-off in the choice of references.
\par \textbf{Trick of reference choice:} Empirically, the explicands with interference objects will be beneficial from the \emph{relevant} references (see Fig. \ref{imagenet_refs}a and Appendix Fig. \ref{append_imagenet_refs}a), the interference objects of which will be less attributed. Contrarily, the image with the explained subject on the pure background will be beneficial from the \emph{irrelevant} references (see Fig. \ref{imagenet_refs}b and Appendix Fig. \ref{append_imagenet_refs}b), where attributions on subjects will be more complete. This rule allows us to sample the references from particular categories to improve the attribution quality.
\par This trick is consistent with the intuition of counterfactual explanation: \emph{the contrast between ambiguous (hard-to-identify) classes will highlight the most critical features of explicands.} Notably, this trick is only necessary for some explicands in specific datasets. In more general cases, we just need to sample the references uniformly from the counterfactual classes.
\par We conduct a quantitative evaluation for different reference choice strategies on the ImageNet dataset, the results of which are reported in Table \ref{ref_quant}. We split the explained samples into two subsets: One subset is with interference objects, and another one is on the pure background. We consider two strategies: ``\emph{Relevant}'' selects references from the relevant categories of the top 4 predictions; ``\emph{Irrelevant}'' randomly and uniformly samples references from other categories.
\par By the denoising trick (using relevant references), the attributions are more concentrated, which leads to higher \emph{ground truth-SUM} and \emph{SIC AUC-DEL} for all images. For the interfered images, all four metrics are significantly improved by this strategy. Since images on the pure background account for the majority in ImageNet ($\sim 72\%$ in our test set), the two reference choice strategies are competitive on the whole test set. Without loss of generality, on ImageNet dataset, IG\textsuperscript{2} uses the strategy that randomly samples the references by default.


\begin{figure}[!t]
  \centering
  \includegraphics[width=0.49\textwidth]{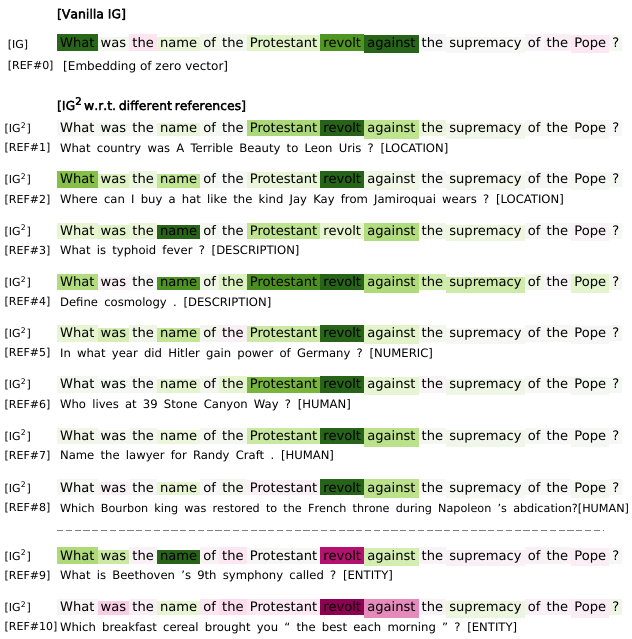}
  \caption{\textbf{Word attributions w.r.t. different references}, for the question ``\emph{What was the name of the Protestant revolt against the supremacy of the Pope?}'' of class [entity]. REFs \#1 to \#8 are used for calculating the Expected IG\textsuperscript{2} reported in Fig. \ref{trec_explain} \#4. }
  \label{trec_ref}
\end{figure}

\begin{figure}[!t]
  \centering
  \setlength{\belowcaptionskip}{-0.4cm}
  \includegraphics[width=0.40\textwidth]{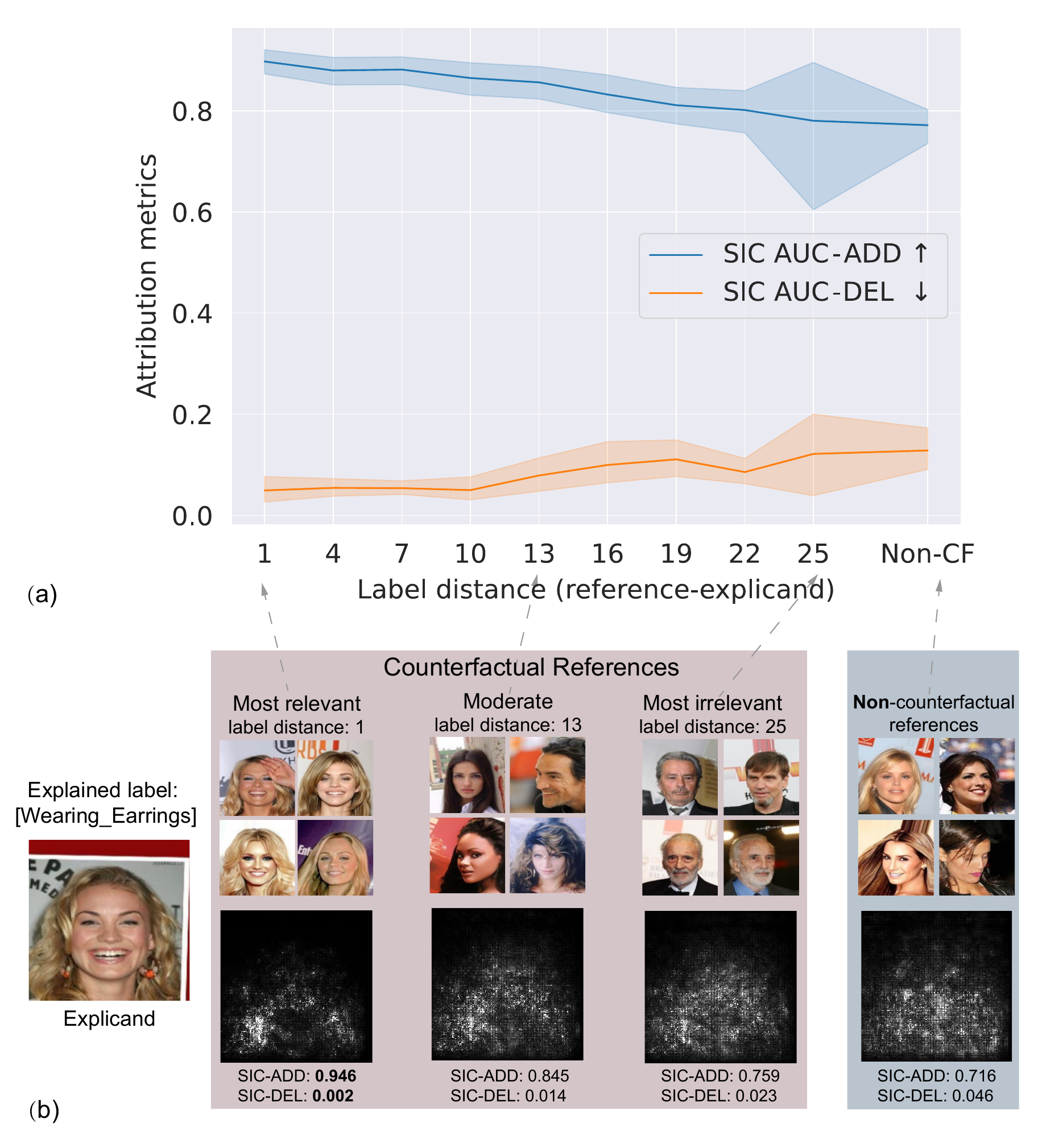}
  \caption{\textbf{(a) Curves of attribution metrics w.r.t. CelebA label distance} between reference and explicand. \textbf{(b) A CelebA example attributed by IG\textsuperscript{2}} of different references (explained label: [Wearing\_Earrings]).}
  \label{celebA_illu}
\end{figure}
\subsubsection{References of TREC}
\label{sec_trec_ref}
\par Fig. \ref{trec_ref} lists IG\textsuperscript{2} attributions w.r.t. different references, for the question \#4 in Fig. \ref{trec_explain}. Fig. \ref{trec_ref} intuitively shows that references of different categories provide different counterfactual contrast, which makes IG\textsuperscript{2} explanation more in line with the human-authored grammar rules.
\par Based on grammar rules, the subject phrase of the question \#4 (of class [entity]) should be ``name of revolt''. The two most attributed words by Vanilla IG are ``what'' and ``against'' (see REF \#0 in Fig. \ref{trec_ref}), which are inaccurate since they do not determine the class [entity].
\par IG\textsuperscript{2} uses references \#1 to \#8, where some insights can be found into the counterfactual explanation:
\begin{itemize}
  \item Compared to IG, all the references make IG\textsuperscript{2} have much less attribution on ``what'', which is irrelevant to the class [entity]. Notably, even if the interrogative word of reference is not ``what'', IG\textsuperscript{2} attribution is still significantly reduced (see REFs \#4 to \#8).
  \item Almost all the counterfactual references contrastively highlight ``\emph{revolt}'', since modifying this word is the fastest way to turn the explicand into another class~\footnote{For instance, replacing ``\emph{revolt}'' with ``\emph{country}'' turns the question into class [location], and replacing with ``\emph{leader}'' turns into [human].}. Reference \#3 is an exception, an interpretation of which is: if we only modify the word ``\emph{revolt}'', the sentence will hardly become the class [description]; but if removing the phrase ``\emph{the name of}'', the explained question will be similar to reference \#3.
  \item The word ``\emph{name}'' is also more highlighted by IG\textsuperscript{2}, but the references of classes [human] and [location] do not provide this contrast. This is because the word ``\emph{name}'' is also related to the classes [human] and [location] (e.g., ``name of lawyer'' and ``name of country'').
\end{itemize}
\par Finally, REFs \#9 and \#10 in Fig. \ref{trec_ref} list the attributions w.r.t. the references of the same class, where the word ``\emph{revolt}'' is highlighted in the opposite direction. The attributions w.r.t. non-counterfactual references are confusing, which is why we only consider the references of different classes when we compute the Expected IG\textsuperscript{2}.


\subsubsection{Reference of CelebA}
\label{sec_ref_celebA}
\par We leverage a multi-label dataset to quantitatively evaluate the choice of references. Thus we can use the $\ell_1$ distance of label vectors to measure the similarity between references and explicand, and then build the quantitative correlation between the reference categories and feature Attributions.

\par  Fig. \ref{celebA_illu}a reports the feature attribution metric curves w.r.t. label distances, averaged over 300 CelebA samples. Fig. \ref{celebA_illu}b displays an explicand with the explained label [Wearing\_Earrings], showing its attributions w.r.t. different references. The non-counterfactual references have the identical explained label to the explicand, and \emph{vice versa}.

\par The results show, the similar counterfactual references (with small label distances) lead to better feature attributions (e.g., highlighting the earring pixels). Higher label distances result in a decline in the attribution metrics. The non-counterfactual references cannot contrast the explained label, which leads to the worst feature attributions.
\par Based on the results, we make two conclusions for the reference choice on the multi-label dataset: (1) References should be counterfactual on the explained label; (2) Using relevant references will benefit the feature attributions of multi-label samples. The second conclusion is consistent with the above denoising tricks for images with interference objects. In the multi-label classification, the not-explained labels (e.g., the face attributes other than [Wearing\_Earrings]) can be regarded as the interference objects, so relevant references will make the feature attributions less noisy and more concentrated on the explained image parts.

\subsection{Step size and number}
\par Appendix Fig. \ref{append_steps} uses an ImageNet example to illustrate the effect of step size and number on IG\textsuperscript{2} attribution. Appendix Fig. \ref{append_steps}a shows the objective loss (Eq. \ref{obj}) curves during the optimization for GradPath, and Appendix Fig. \ref{append_steps}b shows IG\textsuperscript{2} attributions w.r.t. different step sizes and numbers. The step sizes and numbers chosen for different datasets are reported in Appendix Table \ref{hp_ig2}.

\par \textbf{Step size} can affect the IG\textsuperscript{2} attributions. There is a trade-off issue on step size: small step sizes tend to result in less noisy but incomplete attributions, whereas large step sizes result in complete but noisy attributions. This is intuitive: when the total magnitude is more limited, the perturbation will be concentrated on a few more important features. In practice, we will choose a moderate step size that can well balance these two sides.

\par \textbf{Step number} is not a critical hyper-parameter. Appendix Fig. \ref{append_steps} shows that IG\textsuperscript{2} attribution does not substantially change when objective loss approaches convergence. Hence, we set the step number to a relatively large value that can allow optimizations of most explicands to be converged.


\subsection{Representation distance measure}
\par IG\textsuperscript{2} constructs GradPath and searches GradCF by minimizing the distance between model representations of reference and explicand. Eq. \ref{obj} shows this optimization objective, where the distance between two vectors is based on the Euclidean measure. The usage of Euclidean distance is inspired by the feature matching trick in GAN training~\cite{DBLP:journals/corr/SalimansGZCRC16} and the adversarial robustness work~\cite{ilyas2019adversarial}. Besides, the Cosine similarity and $\ell_1$ norm are two common measures for the vector distance. We conduct an ablation study to compare the different distance measures in IG\textsuperscript{2}.
\par Appendix Table \ref{similar_quant} reports the performance gap between Cosine similarity and $\ell_1$ norm to Euclidean distance, and Appendix Fig. \ref{rep_sim} displays ImageNet samples attributed by three different measures. Based on the quantitative and qualitative results, we argue that Euclidean distance is not significantly better than Cosine similarity, while $\ell_1$ norm is not a good choice for IG\textsuperscript{2} attributions. Empirically, IG\textsuperscript{2} uses the Euclidean distance by default, achieving slightly higher evaluation metrics in experiments. 

\subsection{Computational cost}
\par We analyze the computational cost of IG\textsuperscript{2}. The major computational cost of path methods is gradient calculation. Since GradPath requires the same gradient calculation times as the gradient integration, IG\textsuperscript{2} requires at least twice the computational cost over other methods. On the other hand, IG\textsuperscript{2} commonly requires about 10 references to get the desired performance, while other methods only require one. Hence, the number of gradient calculations in IG\textsuperscript{2} is about 20 times more than other methods.
\par In practice, we can reduce the running time by calculating the gradients of different references in one batch. The practical running time of IG\textsuperscript{2} is about 10 to 20 times that of (Expected) IG and about 3 times that of Guided IG. Appendix Table \ref{append_compute_time} reports the average explanation time per sample of different methods. Despite the increased computational cost, the running time of IG\textsuperscript{2} is still practically feasible. Compared with the sampling-based KernelSHAP method, even with superpixel techniques, IG\textsuperscript{2} is significantly faster on the high-dimensional datasets.

\par Besides above hyperparameters, the discussions about normalization and representation layer are presented in Appendix \ref{append_normal} and \ref{append_rep}


\section{Conclusion}
\label{conclusion}
This paper proposes a novel feature attribution method, \underline{I}terative \underline{G}radient path \underline{I}ntegrated \underline{G}radients (IG\textsuperscript{2}), which simultaneously incorporates two gradients, the explicand's and counterfactual. IG\textsuperscript{2} proposes two novel essential components of path methods: baseline (GradCF) and integration path (GradPath). GradPath incorporates the counterfactual gradient into its direction and implicitly solves the saturation effects and noisy attributions. GradCF is the first baseline that contains the information of both model and explicand, avoiding the previous arbitrary baseline choice.
\par We contrast our work with path methods and the works in the field of counterfactual explanation and adversarial learning. We argue that our work can be regarded as a generalization of Guided IG. We intuitively interpret our proposal and justify the desirable axioms of IG\textsuperscript{2} in theory. The effectiveness of IG\textsuperscript{2} is verified by an XAI benchmark and multiple real-world datasets from diverse domains with qualitative and quantitative results. Moreover, the ablation study reveals that GradPath and GradCF individually improve the attribution of IG methods, harmonized by IG\textsuperscript{2}.

\bibliographystyle{IEEEtran}
\bibliography{IEEEabrv,bibliography}
\appendices

\renewcommand\thefigure{\thesection.\arabic{figure}}  
\setcounter{figure}{0}
\renewcommand\thetable{\thesection.\arabic{table}}  
\setcounter{table}{0}

\section{Implementation details}
\label{append_imple}

\subsection{Step size and number}
\par Fig. \ref{append_steps}a shows the objective loss (Eq. \ref{obj} in main paper) curves during the optimization for GradPath, and Fig. \ref{append_steps}b shows IG\textsuperscript{2} attributions w.r.t. different step sizes and numbers. 
\par We choose a moderate step size that can deal well with the trade-off between attribution noise and completeness. The specific hyper-parameters about step size and number for different datasets are reported in Table \ref{hp_ig2}.

\subsection{Normalization}
\label{append_normal}
\par We consider two normalization methods: $\ell_2$ norm and $\ell_1$ norm. Given the gradient $g_\alpha=\frac{\partial \Vert \tilde{f}(\gamma^{G}(\alpha)) - \tilde{f}(\alpha) \Vert_2 }{ \partial \gamma^{G}(\alpha)}$ at step $\alpha$, we can have $W^{\ell_2}_\alpha$ and $W^{\ell_1}_\alpha$:

\par \textbf{$\ell_2$ norm} is simple:
\begin{align}
  W^{\ell_2}_\alpha = \Vert g_\alpha \Vert_2,
\end{align}
which is used in previous works about adversarial robustness~\cite{madry2018towards,ilyas2019adversarial}.
\par \textbf{$\ell_1$ norm} is more tricky. Following a sparse adversarial attack work~\cite{sparseattack}, an additional parameter $k$ that controls the sparsity of gradient on each step is introduced. Specifically, we select the features $i$ with top $k$ maximal absolute gradient value and then normalize the sparse vector with sign function:
\begin{align}
  \label{l1norm}
  W^{\ell_1}_\alpha(i) =\begin{cases}
                          \vert S \vert g(i),\  i \in \mathcal{S}=\underset{\vert A \vert=k}{\mathrm{argmax}\ \Sigma_{i\in A} \vert g(i) \vert} \\
                          0
                        \end{cases},
\end{align}
where $S$ is the set contains the index of top $k$ maximal absolute value in vector $g$ and the set size $\vert S\vert$ is the $\ell_1$ norm of signed vector.
\par Both $l_1$ and $l_2$ normalization imply the feature importance on each step perturbation, the comparison of which is presented below. In practice, we use more general $\ell_2$ normalization by default.

We compare $\ell_2$ and $\ell_1$ normalization for the gradient of GradPath (i.e., $W$). $\ell_1$ normalization has an additional parameter to control sparsity, which is presented by the percentile of the gradient vector. Fig. \ref{l1l2_comparison} plots the IG\textsuperscript{2} attributions on wafer map under different normalizations. When the percentile is low (90), the attribution under $\ell_1$ normalization resembles $\ell_2$. When the attribution sparsity increases, the noises significantly drop, but some attributions on relevant features also get very small, which especially damages the performance of certain classes (e.g., Donut, Near-full, and Random). This can be regarded as a precision and recall trade-off.

\begin{table}[h]
  \centering

  \caption{IG\textsuperscript{2} hyper-parameters for different datasets }
  \begin{tabular}{ccccc}
  \toprule
  Datasets & Input region & Step size   & Steps & Ref.         \\ \midrule
  ImageNet & [0,255] &1024& 500  & random (trick) \\
  TREC & [-1.156, 1.094] & 0.01 & 1000  & random  \\ 
  wafer map& [0,1] & 0.02& 2000 & the normal \\\bottomrule
  \end{tabular}
  \label{hp_ig2}
  \end{table}
  \begin{figure}[!t]
    \centering

    \includegraphics[width=0.40\textwidth]{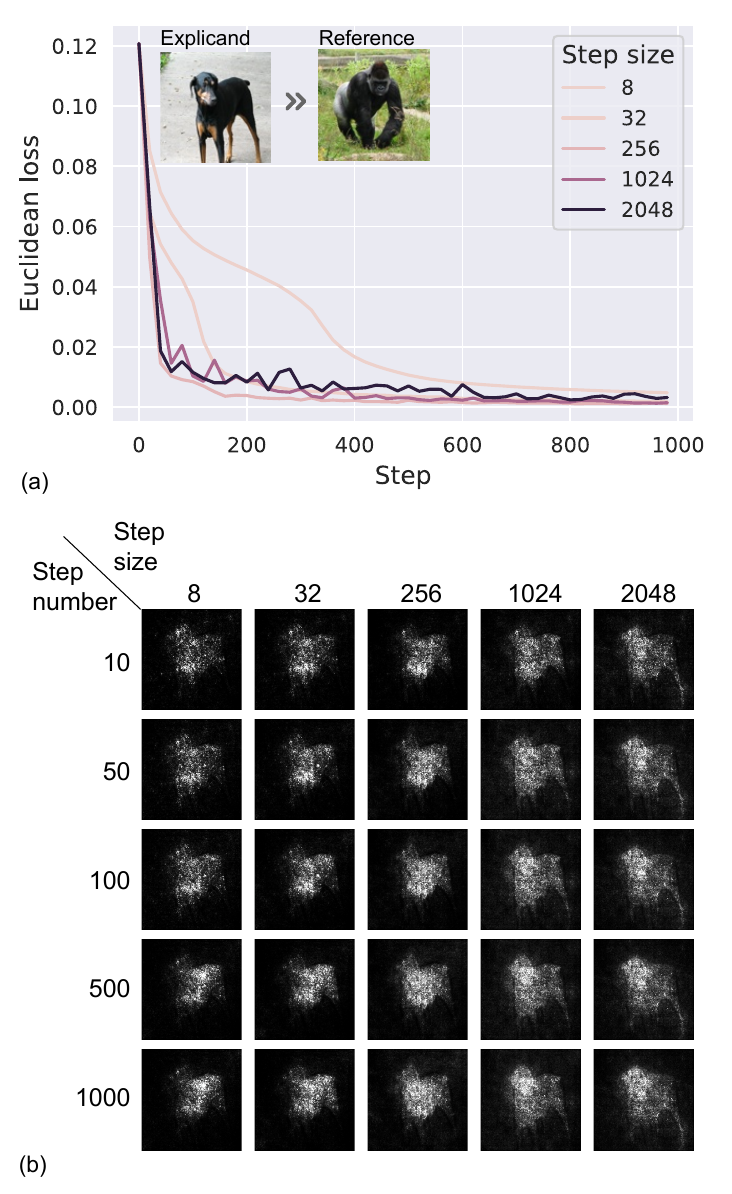}
    \caption{\textbf{(a) Curves of loss w.r.t. step numbers} for different step sizes. \textbf{(b) IG\textsuperscript{2} attributions w.r.t. step sizes and numbers.} [Doberman] is explained with reference [gorilla].}
    \label{append_steps}
  \end{figure}

\begin{figure}[]

  \centering
  \includegraphics[width=0.45\textwidth]{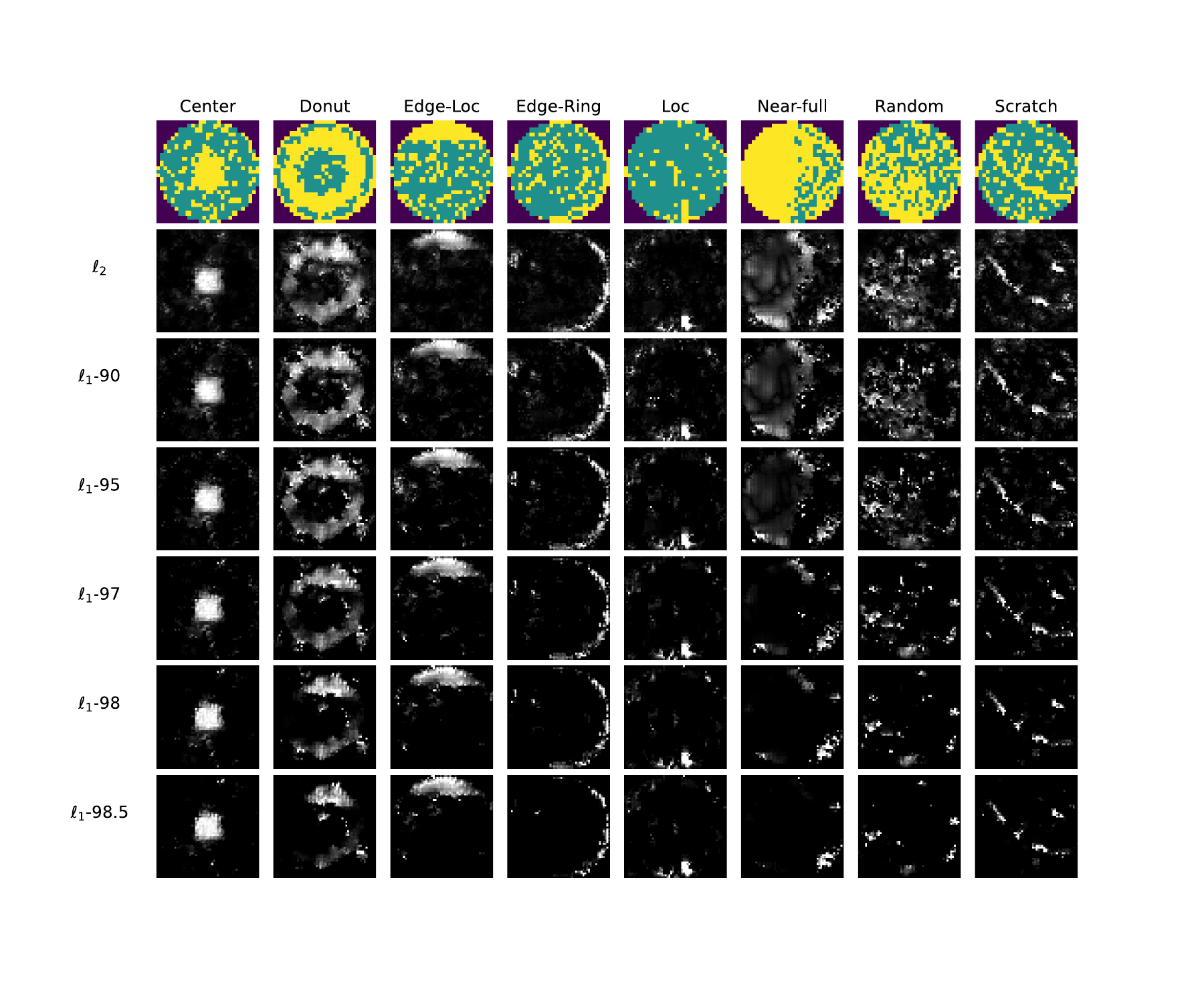}
  \caption{\textbf{IG\textsuperscript{2} attributions of $\ell_1$ and $\ell_2$ normalization.} Five sparsity percentile parameters (90, 95, 97, 98, 98.5) of $\ell_1$ normalization are compared. }
  \label{l1l2_comparison}
\end{figure}

\begin{figure}[]

  \centering
  \includegraphics[width=0.45\textwidth]{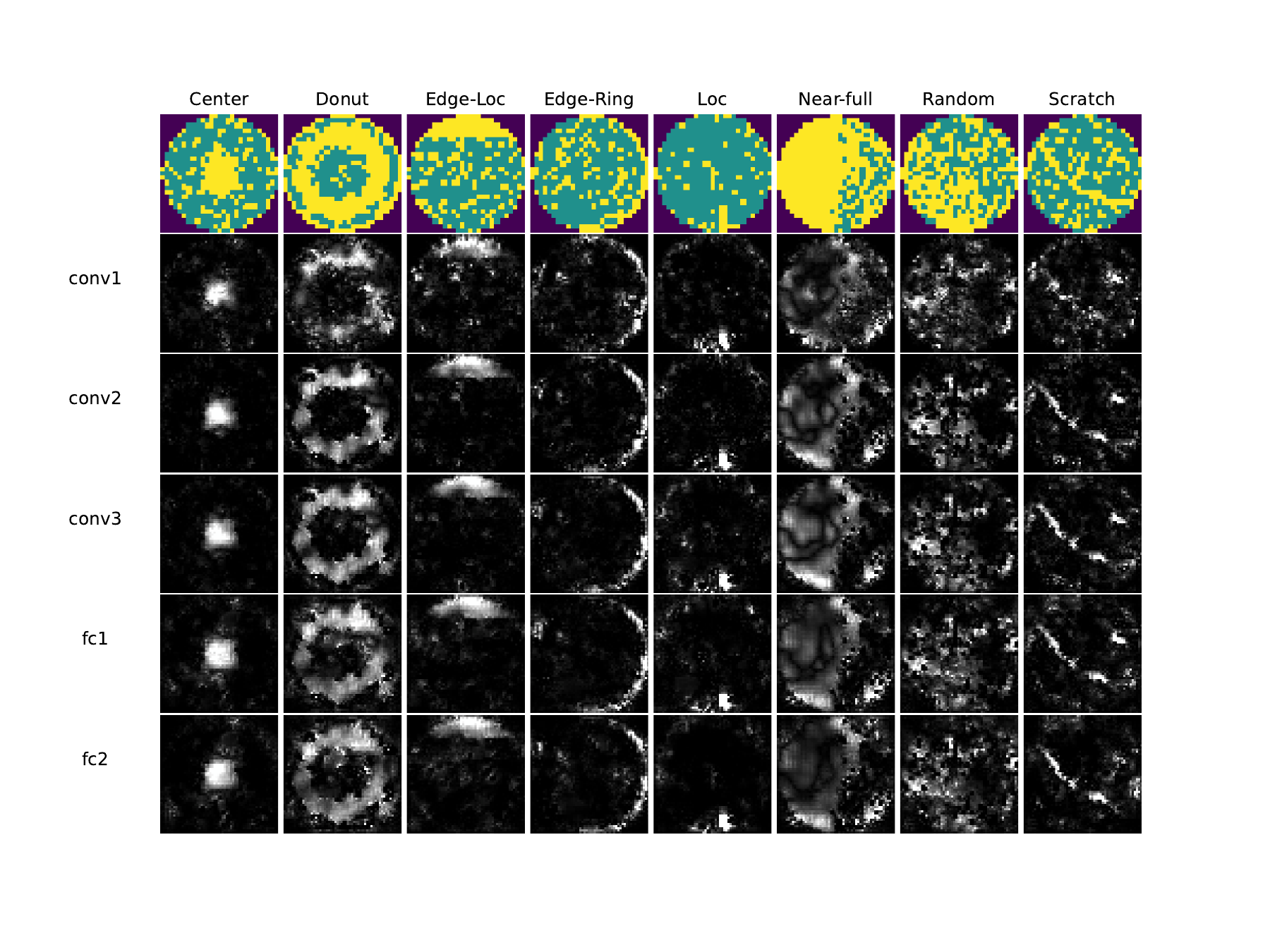}
  \caption{\textbf{IG\textsuperscript{2} with the different representation layer} (fc1 used in the paper). }
  \label{rep_ig2}
\end{figure}
\subsection{Representation layer}
\label{append_rep}
We compare the IG\textsuperscript{2} with the different representation layers of wafer map classifier (Table \ref{wafermap_cnn}), which is shown in Fig. \ref{rep_ig2}. GradCFE gets less noise using the deeper layer as representation, while IG\textsuperscript{2} is slightly affected by the choice of representation layer.

\subsection{Representation distance measure}
\label{append_repsim}
We compare the representation distance measures of IG\textsuperscript{2}. Fig. \ref{rep_sim} shows the ImageNet sample attributions by three measures, as an intuitive illustration of main paper Table \ref{similar_quant}. Generally, IG\textsuperscript{2} with Euclidean distance and Cosine similarity provide very similar feature attributions, while $\ell_l$ norm will introduce unpleasant noises.

\begin{table}[!h]
  \centering
  \caption{Evaluation of model representation distance measures (the gap to Euclidean distance)}
  \label{similar_quant}
  \begin{threeparttable}

    \begin{tabular}{@{}llll|ll@{}}
      \toprule
                                &               & \multicolumn{2}{c|}{Ground truth} & \multicolumn{2}{c}{SIC AUC}                                     \\
      Datasets                  & Measures      & AUC $\uparrow$                    & SUM $\uparrow$              & ADD $\uparrow$ & DEL $\downarrow$ \\ \midrule
      \multirow{2}{*}{ImageNet} & Cosine        & $-0.055$                     & $-0.021$               & $+0.018$  & $+0.047$    \\
                                & $\ell_1$ norm & $-0.065$                    & $-0.027$               & $-0.028$  & $+0.003$    \\ \midrule
      \multirow{2}{*}{TREC}   & Cosine          &\multicolumn{2}{c|}{\multirow{2}{*}{---}}   & $+0.009$          & $+0.007$            \\
                                & $\ell_1$ norm &    \multicolumn{2}{c|}{}                   & $-0.010$          & $+0.012$            \\ \midrule
      \multirow{2}{*}{\shortstack{Wafer \\map}}      & Cosine             & $-0.024$                             & $-0.019$                       & $-0.010$          & $-0.002$            \\
                                & $\ell_1$ norm & $-0.072$                             & $-0.037$                       & $-0.029$          & $+0.016$            \\ \bottomrule
    \end{tabular}
  \end{threeparttable}
\end{table}

\begin{figure}[!h]
  \setlength{\belowcaptionskip}{-0.4cm}
  \centering
  \includegraphics[width=0.49\textwidth]{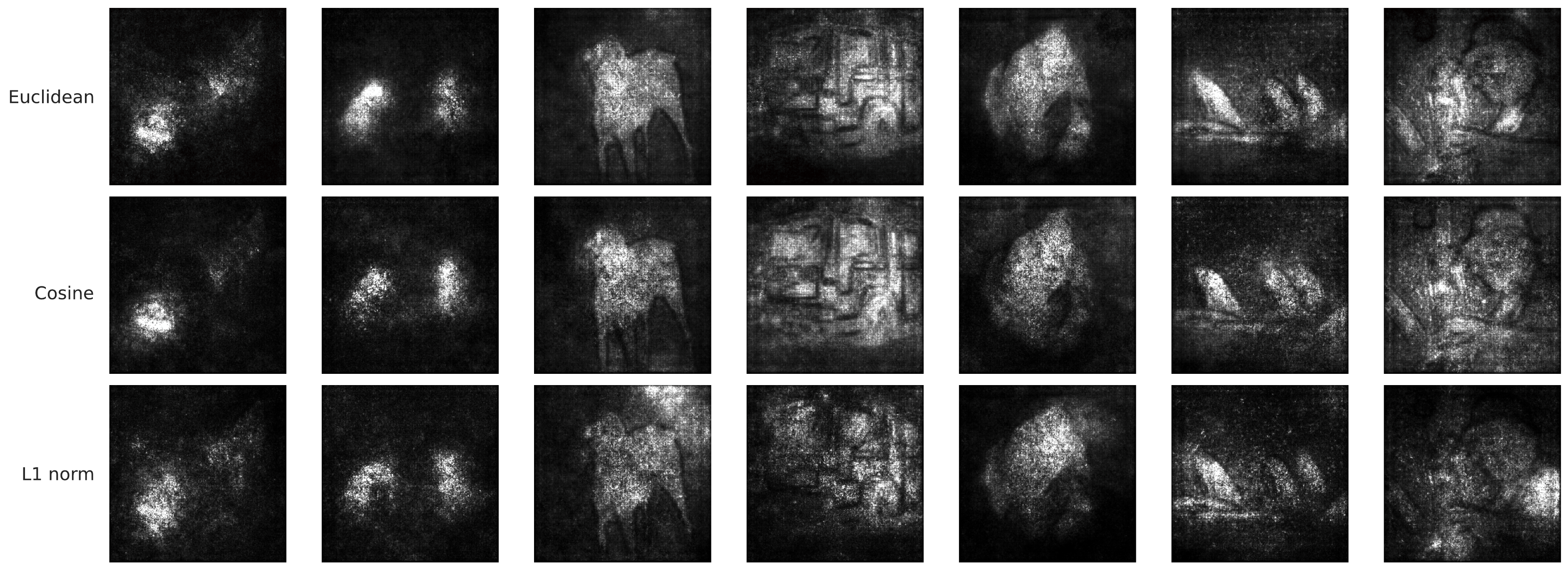}
  \caption{\textbf{IG2 attributions of Euclidean, Cosine, and $\ell_1$ norm measures} (Euclidean distance used in the paper). }
  \label{rep_sim}
\end{figure}

\subsection{Computational cost}
\label{append_compute_cost_sec}
\par Table \ref{append_compute_time} reports the average computation time per explained sample. For a fair comparison, all the methods run with the same step number and batch size (if available). We run the experiments with 1 NVIDIA GeForce GTX 1080 Ti GPU. The KernelSHAP applies the superpixel technique for ImageNet and wafer map datasets.
\begin{table}[t]
  \centering

  \caption{Average explanation time per sample (sec.). The input feature numbers of KernelSHAP are reduced. }
  \label{append_compute_time}
  \begin{tabular}{@{}lccc@{}}
  \toprule
           & ImageNet & TREC  & wafermap \\ \midrule
  Gradient        & 0.13     & 0.003 & 0.03     \\
  Vanilla IG       & 8.14     & 0.26  & 0.89     \\
  Guided IG         & 37.0     & 4.48  & 6.39     \\
  Expected IG      & 8.34     & 0.49  & 1.25     \\
  DeepSHAP          & 111.0      & 0.32  &  0.95    \\ 
  KernelSHAP           & 3150.2 ($13\times13$) & 1.8  &  75.0 ($28\times28$)    \\ 
  IG\textsuperscript{2}       & 108.2    & 9.80  & 18.3     \\ 
  \bottomrule
  \end{tabular}
  \end{table}

\begin{table}[t]
  \centering
  \caption{The structure of MLP to be explained for the regression on the synthetic dataset from XAI-BENCH }
  \begin{tabular}{|c|c|}
  \hline
  \textbf{Layer} & \textbf{Configuration}            \\ \hline
  input & $5$-Dimension vector       \\ \hline
  fc1 & fc\textsubscript{5,64}+BN+Tanh  \\ \hline
  fc2 & conv\textsubscript{64,16}+Tanh  \\ \hline
  fc3 & conv\textsubscript{16,1} \\ \hline
  \end{tabular}
  \label{XAI_MLP}

  \end{table}

  \begin{figure}[!h]
        \centering
        \includegraphics[width=0.45\textwidth]{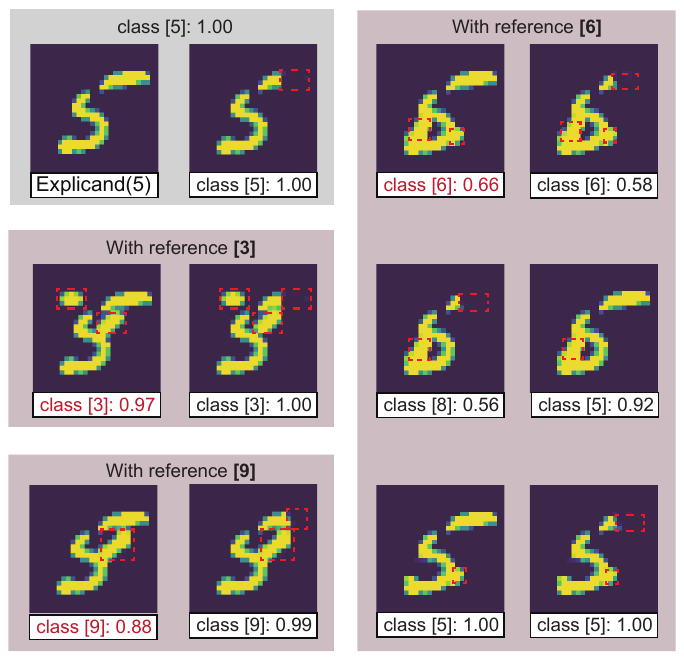}
        \caption{\textbf{Ablation study for critical areas} in the explained digital 5 with references of digitals 3, 6, 9. The titles of each subfigure report the predicted class with Softmax output. The areas in the first figure of each block (with red titles) are used in main paper Fig. \ref{mnist_example}.  }
        \label{mnist_ablation}
      \end{figure}

  \section{Datasets and explained models}
  \label{append_datamodel}
  \subsection{MNIST classification}
  \label{append_MNIST}
  \par Fig. \ref{mnist_ablation} validates the most critical areas in digital 5 that distinguishes the explicand to references of digitals 3, 6, 9. We compare the areas highlighted by IG\textsuperscript{2} (the areas in the figures with red titles) and the areas of Expected IG (the right upper corners). We study the importance of different areas by the permutations: We first remove the right upper pixels of digital 5, which does not change model prediction. If we fill the highlighted areas of IG\textsuperscript{2} with pixels, the model predicts the modified digitals as the class labels of references. Based on this, if we further remove the right upper corners of these digitals, the impact on model predictions is not significant.
  \par In sum, compared with the right upper corners of Expected IG, the areas highlighted by IG\textsuperscript{2} are more critical for distinguishing the explicand to references. Besides, we believe areas of IG\textsuperscript{2} are also closer to the intuition. 

  \begin{figure}[t]
    \setlength{\belowcaptionskip}{-0.1cm}
    \centering
    \includegraphics[width=0.48\textwidth]{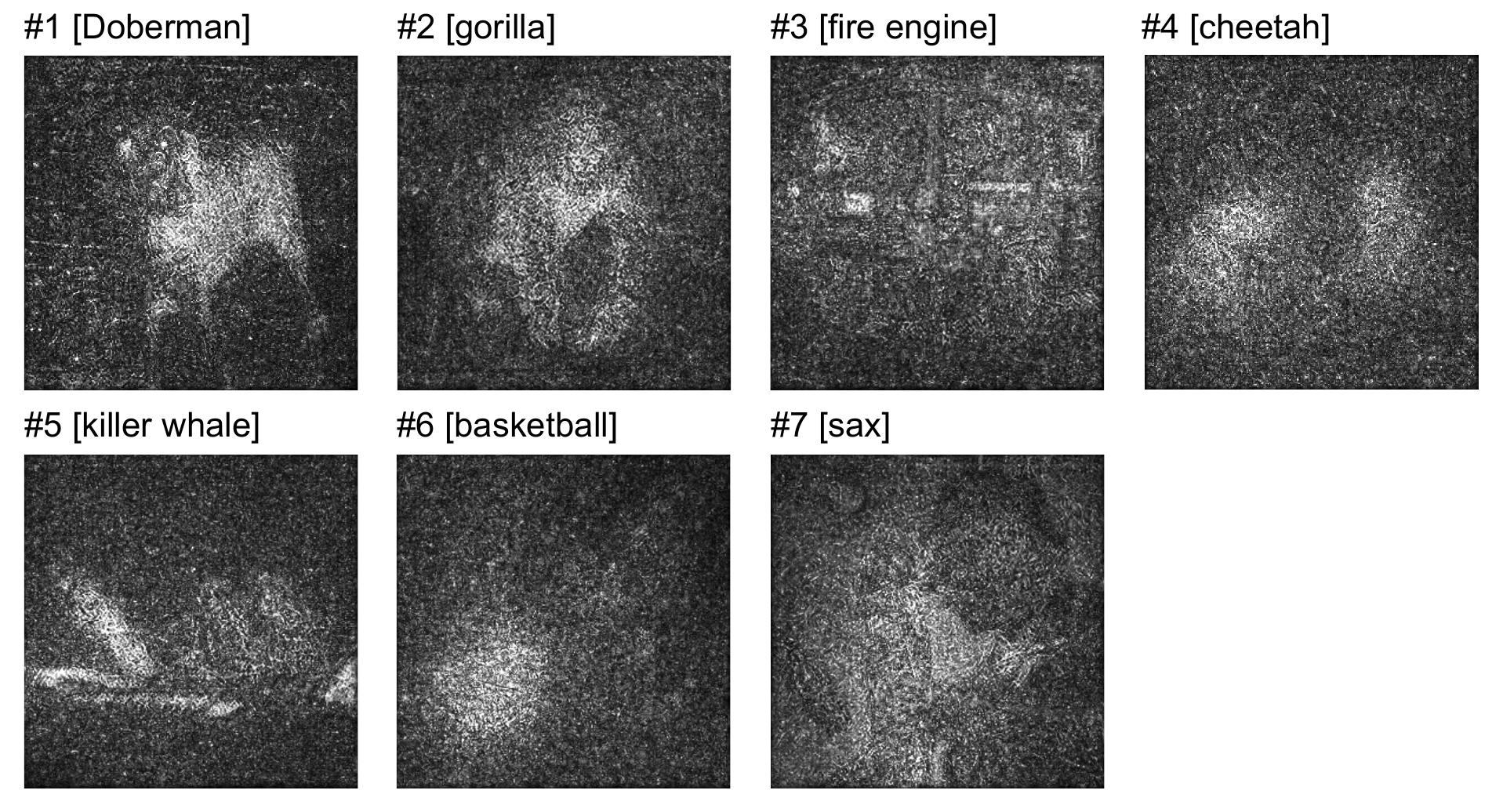}
    \caption{\textbf{GradCFE for the ImageNet images in main paper Fig. \ref{imagenet_explain},} showing the general counterfactual directions of GradPath (from the explicand to its GradCF).}
    \label{imagenet_gracfe}
  \end{figure}
  
  \begin{figure}[!h]
    \centering
    \includegraphics[width=0.48\textwidth]{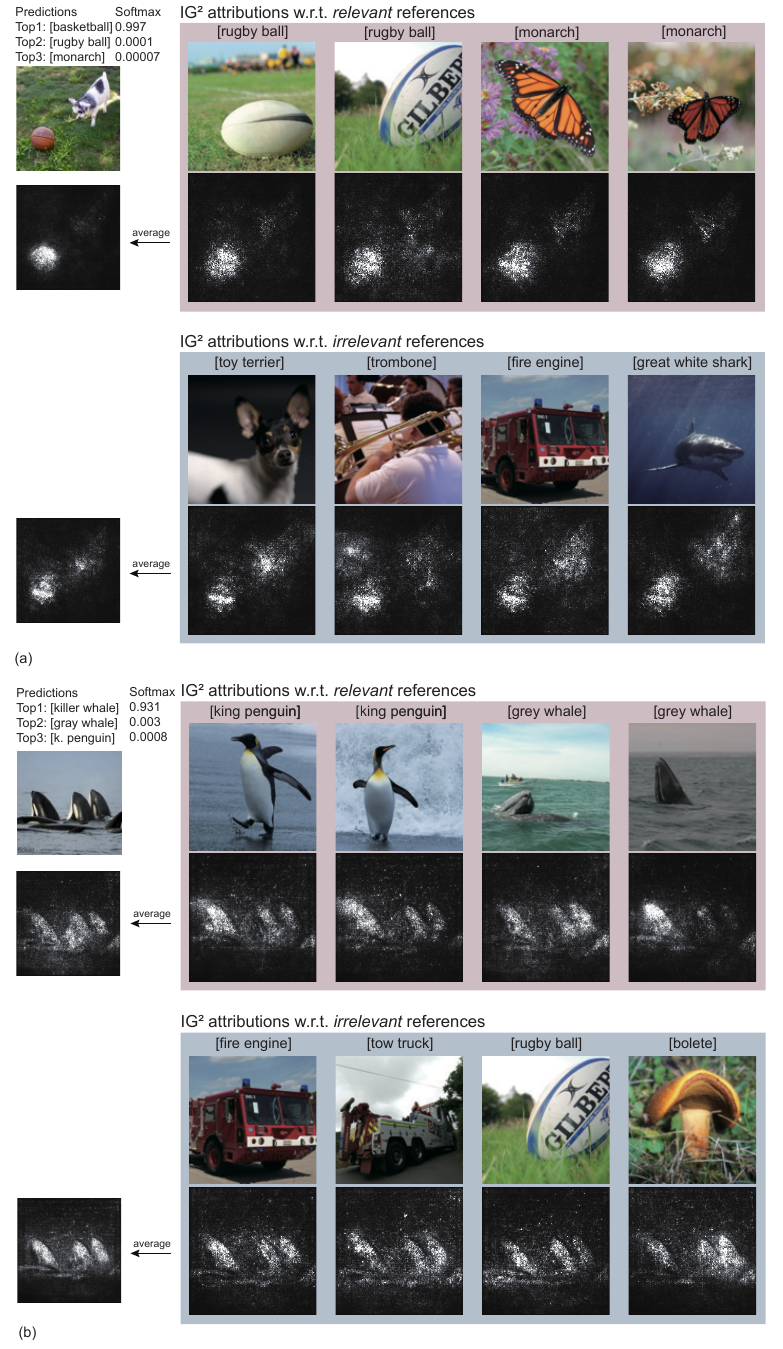}
    \caption{\textbf{Feature attributions with regard to different references,} as a supplement for main paper Fig. \ref{imagenet_refs}. \textbf{(a) Explained image of class [basketball]} with interference object ``dog''. \textbf{(b) Explained image of class [killer whale]} on pure background. }
    \label{append_imagenet_refs}
  \end{figure}

  \subsection{Synthetic dataset in XAI-BENCH}
  \label{append_datamodel_xai}
  \par We synthesize the dataset $x$ from the 5-dimension Gaussian distribution with $\mathbf{\mu}=\vec{0}$ and $\mathbf{\Sigma}=I_5$. The additive piecewise function for labels are~\cite{XAI_bench}:
  \begin{align}
    \psi_1(x_1) &= 1\ \text{if}\ x_1 \ge 0,\ \text{otherwise}\ -1 \notag ,\\
    \psi_2(x_2) &= \begin{cases}
      -2&,\ x_2<-0.5\\
      -1&,\ -0.5 \le x_2 <0 \\
      1&,\ 0 \le x_2 < 0.5 \\
      2&,\ x_2 \ge 0.5
    \end{cases}, \notag \\
    \psi_3(x_3) &=  \lfloor 2cos(\pi x_3) \rfloor , \notag \\
    \psi_i(x_i) &= 0,\ i=4,5 ,
  \end{align}
  and we cut off the normalized function output with a threshold to obtain the binary label:
  \begin{align}
    y = \begin{cases}
      1&,\ \text{if}\ \mathrm{Norm}[\Sigma_{i=1}^5 \psi_i]  > 0\\
      0&,\ \text{if}\ \mathrm{Norm}[\Sigma_{i=1}^5 \psi_i ]\le 0
    \end{cases},
  \end{align}
  where $\mathrm{Norm}[\ ]$ makes the mean value of outputs zero.
  \par The structure Multi-Layer Perceptron (MLP) trained on the synthetic dataset is reported in Table \ref{XAI_MLP}.

  \begin{figure}[t]
    \centering
    \includegraphics[width=0.48\textwidth]{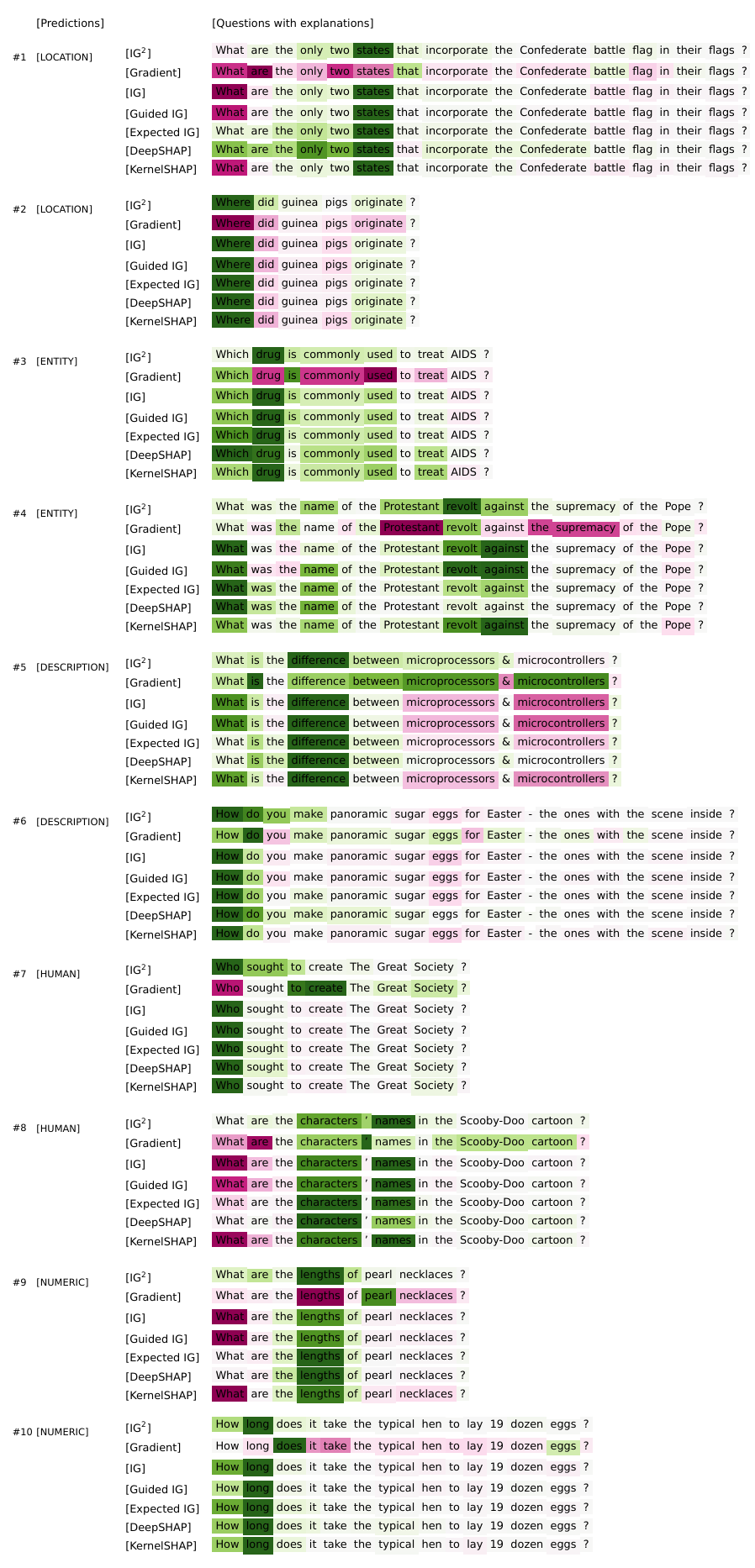}
    \caption{\textbf{Word attributions of all compared methods for TREC dataset,} as a complement for main paper Fig. \ref{trec_explain}.}
    \label{trec_explain_supply}
  \end{figure}

  \subsection{ImageNet classification}

  \par Fig. \ref{imagenet_gracfe} shows the GradCFE for the explicand displayed in main paper Fig. \ref{imagenet_explain}. GradCFE presents the difference between the explicand and its baseline, GradCF. Fig. \ref{imagenet_gracfe} illustrates that the integration path of IG\textsuperscript{2} is in the direction of critical features, which can mitigate the saturation effects.
  \par As a supplement for main paper Fig. \ref{imagenet_refs}, Fig. \ref{append_imagenet_refs} shows the feature attributions with regard to different categories of references for two more ImageNet samples. This supports the analyses of reference impact and choice trick.

  \subsection{TREC classification}
    \par Fig. \ref{trec_explain_supply} shows the complete word attribution results of all compared methods. The results are consistent with the conclusions made in the main paper that IG\textsuperscript{2} makes less attributions on weak interrogative words and more attributions on critical phrases.

  \subsection{Wafermap failure pattern classification}
  \label{append_datamodel_wafer}
  Fig. \ref{Wafermap_overall} plots the wafer maps sampled from WM-811K dataset~\cite{6932449}, including three images per class. To achieve the task of failure pattern classification, we select a small subset of WM-811K dataset with a total of 19000 wafer maps. The backbone of the classification neural network is based on convolution, the detailed structure of which is reported in Table \ref{wafermap_cnn}.
  
  \begin{figure}[t]
    \centering
    \includegraphics[width=0.48\textwidth]{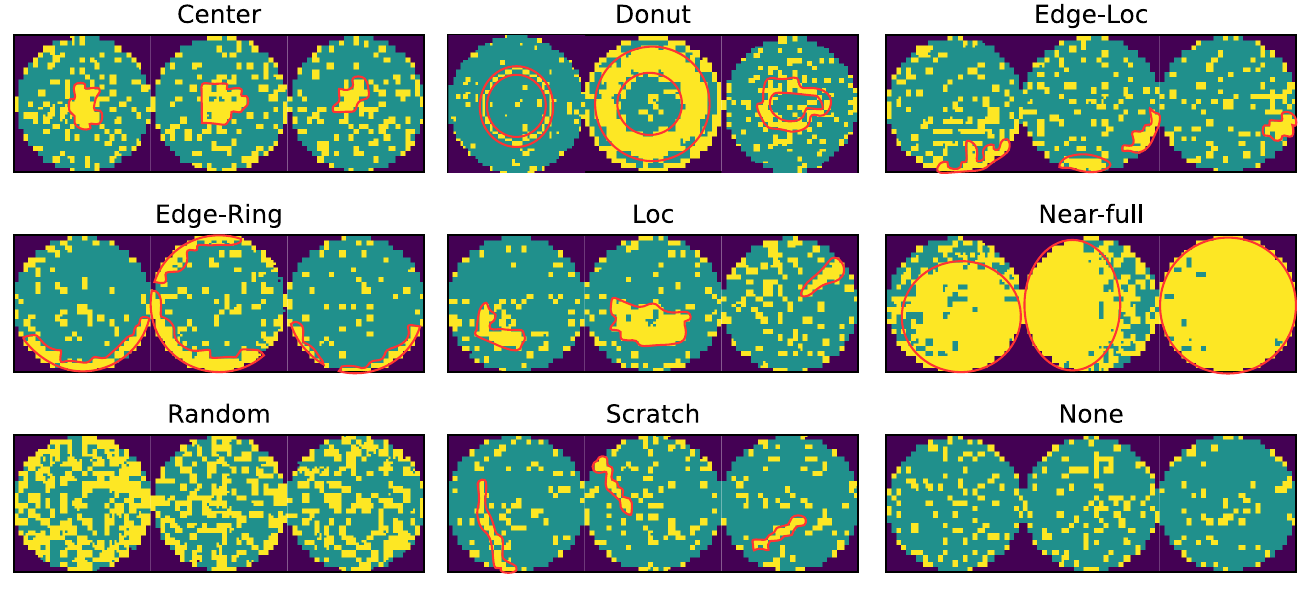}
    \caption{The sampled wafer map images from WM-811K dataset, including eight failure patterns and one without any pattern: [Center], [Donut], [Edge-Loc], [Edge-Ring], [Loc], [Near-full], [Random], [Scratch] and [None]. The failure cause position (pixels) are also annotated in the red border area (except [Random] and [None]). }
    \label{Wafermap_overall}
  \end{figure}
  \begin{table}[t]
    \centering
    \caption{The structure of CNN to be explained for wafer map failure pattern classification }
    \begin{tabular}{|c|c|}
    \hline
    \textbf{Layer} & \textbf{Configuration}            \\ \hline
    input & $56\times 56\times 3$ RGB images       \\ \hline
    conv1 & conv\textsubscript{3,32,3}+ReLU+Pool  \\ \hline
    conv2 & conv\textsubscript{32,64,3}+ReLU+Pool  \\ \hline
    conv3 & conv\textsubscript{64,128,3}+ReLU+Pool \\ \hline
    fc1   & fc\textsubscript{8192,1250}+ReLU      \\ \hline
    fc2   & fc\textsubscript{1250,9}+Softmax      \\ \hline
    \end{tabular}
    \label{wafermap_cnn}
    \end{table}

  \subsection{CelebA face attribute classification}
  \label{append_clelba}
  Face attribute classification in CelebA is a multi-label task. There are totally 40 face attributes with binary labels. The classification accuracy for each face attributes are reported in Table \ref{celebA_acc}.
  \begin{table}[!h]
    \centering
    \caption{Classification accuracy on each face attribute of CelebA  }
    \begin{tabular}{lclc}
    \toprule
    \textbf{Attributes} & \textbf{Acc.} (\%) & \textbf{Attributes}   & \textbf{Acc.} (\%) \\ \midrule
    5\_o\_Clock\_Shadow & 94.4          & Male                  & 97.9          \\
    Arched\_Eyebrows    & 83.4          & Mouth\_Slightly\_Open & 93.6          \\
    Attractive          & 82.5          & Mustache              & 96.9          \\
    Bags\_Under\_Eyes   & 85.0          & Narrow\_Eyes          & 87.7          \\
    Bald                & 98.9          & No\_Beard             & 96.0          \\
    Bangs               & 95.9          & Oval\_Face            & 75.4          \\
    Big\_Lips           & 71.3          & Pale\_Skin            & 97.0          \\
    Big\_Nose           & 84.6          & Pointy\_Nose          & 77.1          \\
    Black\_Hair         & 89.0          & Receding\_Hairline    & 93.7          \\
    Blond\_Hair         & 95.6          & Rosy\_Cheeks          & 94.9          \\
    Blurry              & 95.9          & Sideburns             & 97.7          \\
    Brown\_Hair         & 88.2          & Smiling               & 92.7          \\
    Bushy\_Eyebrows     & 92.6          & Straight\_Hair        & 83.4          \\
    Chubby              & 95.3          & Wavy\_Hair            & 83.4          \\
    Double\_Chin        & 96.2          & Wearing\_Earrings     & 90.4          \\
    Eyeglasses          & 99.6          & Wearing\_Hat          & 99.0          \\
    Goatee              & 97.5          & Wearing\_Lipstick     & 93.7          \\
    Gray\_Hair          & 98.3          & Wearing\_Necklace     & 87.5          \\
    Heavy\_Makeup       & 91.3          & Wearing\_Necktie      & 96.7          \\
    High\_Cheekbones    & 86.8          & Young                 & 88.2          \\ \midrule
                        &                   & Average               & 91.0          \\ \bottomrule
    \end{tabular}
    \label{celebA_acc}
    \end{table}
\par For the ground truth in quantitative evaluation of celebA attribution (in main paper Table \ref{metrics}), we utilize the face parsing model pretrained on CelebAMask-HQ~\cite{CelebAMask-HQ} to split the face image into different facial components, such as hair, nose and eyes. The ground truth is chosen as the facial components that are most related to the explained face attribute. Table \ref{celebA_corr} gives the correspondences between explained face attributes (of CelebA) and parsed facial components (of CelebAMask-HQ). In quantitative evaluation, we only consider the face attributes that have explicit corresponding components.
\begin{table}[!h]
  \centering
  \caption{Correspondences between explained face attributes and parsed facial components}
  \begin{tabular}{lclc}
    \toprule
    \textbf{Attributes} & \textbf{Comp.}  & \textbf{Attributes}   & \textbf{Comp.} \\ \midrule
    5\_o\_Clock\_Shadow & -                     & Male                  & -                     \\
    Arched\_Eyebrows    & brows               & Mouth\_Slightly\_Open & mouth                     \\
    Attractive          & -                        & Mustache              & -                     \\
    Bags\_Under\_Eyes   & eyes                & Narrow\_Eyes          & eyes                     \\
    Bald                & hair                     & No\_Beard             & -                     \\
    Bangs               & hair                     & Oval\_Face            & -                     \\
    Big\_Lips           & lips                & Pale\_Skin            & skin                     \\
    Big\_Nose           & nose                     & Pointy\_Nose          & nose                     \\
    Black\_Hair         & hair                     & Receding\_Hairline    & hair                     \\
    Blond\_Hair         & hair                     & Rosy\_Cheeks          & -                     \\
    Blurry              & -                        & Sideburns             & hair                     \\
    Brown\_Hair         & hair                     & Smiling               & mouth                     \\
    Bushy\_Eyebrows     & brows               & Straight\_Hair        & hair                     \\
    Chubby              & -                        & Wavy\_Hair            & hair                     \\
    Double\_Chin        & -                     & Wearing\_Earrings     & ear\_r                     \\
    Eyeglasses          & eye\_g                    & Wearing\_Hat          & hat                     \\
    Goatee              & -                     & Wearing\_Lipstick     & lips                     \\
    Gray\_Hair          & hair                     & Wearing\_Necklace     & neck\_l                     \\
    Heavy\_Makeup       & -                        & Wearing\_Necktie      & cloth                     \\
    High\_Cheekbones    & -                     & Young                 & -                    \\ \bottomrule
  \end{tabular}
  \label{celebA_corr}
\end{table}
\end{document}